\documentclass{article}

\usepackage{amsmath}
\usepackage{amsfonts}
\usepackage{xcolor}
\usepackage{amsthm}
\usepackage{soul}

\usepackage{microtype}
\usepackage{subfigure}
\usepackage{booktabs} 

\usepackage{tikz}
\tikzstyle{every picture}+=[remember picture]

\usepackage{graphicx}
\usepackage{epstopdf}

\usepackage{hyperref}
\usepackage[capitalise,noabbrev]{cleveref}

\usepackage[textwidth=20mm,textsize=scriptsize]{todonotes}

\newcommand*{\transpose}{\top} 

\usepackage[accepted]{icml2019}

\icmltitlerunning{Partially Exchangeable Networks and ABC}

\begin{document}
\newtheorem{theorem}{Theorem}
\newtheorem{definition}{Definition}
\twocolumn[
\icmltitle{Partially Exchangeable Networks
           and Architectures for Learning Summary Statistics in Approximate Bayesian Computation}



\icmlsetsymbol{equal}{*}

\begin{icmlauthorlist}
\icmlauthor{Samuel Wiqvist$^*$}{lu}
\icmlauthor{Pierre-Alexandre Mattei$^*$}{itu}
\icmlauthor{Umberto Picchini}{gu}
\icmlauthor{Jes Frellsen}{itu}

\end{icmlauthorlist}

\icmlaffiliation{lu}{Centre for Mathematical Sciences, Lund University, Lund, Sweden}
\icmlaffiliation{itu}{Department of Computer Science, IT University of Copenhagen, Copenhagen, Denmark}
\icmlaffiliation{gu}{Department of Mathematical Sciences, Chalmers University of Technology\\ and the University of Gothenburg, Gothenburg, Sweden}

\icmlcorrespondingauthor{Jes Frellsen}{jefr@itu.dk}

\icmlkeywords{Approximate Bayesian computation, DeepSets, exchangeable data, summary statistics}

\vskip 0.3in

]



\printAffiliationsAndNotice{\icmlEqualContribution} 

\begin{abstract}

We present a novel family of deep neural architectures, named partially exchangeable networks (PENs) that leverage probabilistic symmetries. By design, PENs are invariant to block-switch transformations, which characterize the partial exchangeability properties of conditionally Markovian processes. Moreover, we show that any block-switch invariant function has a PEN-like representation. The DeepSets architecture is a special case of PEN and we can therefore also target fully exchangeable data. We employ PENs to learn summary statistics in approximate Bayesian computation (ABC). When comparing PENs to previous deep learning methods for learning summary statistics, our results are highly competitive, both considering time series and static models. Indeed, PENs provide more reliable posterior samples even when using less training data.

\end{abstract}

\section{Introduction}

We propose a novel neural network architecture to ease the application of approximate Bayesian computation (ABC), a.k.a.~\textit{likelihood-free} inference. The architecture, called partially exchangeable network (PEN), uses partial exchangeability in Markovian data, allowing us to perform ABC inference for time series models with Markovian structure. Empirically, we also show that we can target non-Markovian time series data with PENs. Since the DeepSets architecture \cite{zaheer2017deep} turns out to be a special case of PEN, we can also perform ABC inference for static models. Our work is about automatically construct summary statistics of the data that are informative for model parameters. This is a main challenge in the practical application of ABC algorithms, since such summaries are often \textit{handpicked} (i.e.~ad-hoc summaries are constructed from model domain expertise), or these are automatically constructed using a number of approaches as detailed in \cref{sec:abc}. Neural networks have been previously used to automatically construct summary statistics for ABC. \citet{jiang2017learning} and \citet{creel2017neural} employ standard multilayer perceptron (MLP) networks for learning the summary statistics. \citet{chan2018likelihood} introduce a network that exploits the exchangeability property in exchangeable data. Our PEN architecture is a new addition to the tools for automatic construction of summary statistics, and PEN produces competitive inference results compared to \citet{jiang2017learning}, which in turn was shown outperforming the semi-automatic regression method by \citet{fearnhead2012constructing}. Moreover, our PEN architecture is more data efficient and when reducing the training data PEN outperforms \citet{jiang2017learning}, the factor of reduction being of order $10$ to $10^2$ depending on cases.

Our \textbf{main contributions} are:

\begin{itemize}
    \item Introducing the partially exchangeable networks (PENs) architecture;
    \item Using PENs to automatically learn summary statistics for ABC inference. We consider both static and dynamic models. In particular, our network architecture is specifically designed to learn summary statistics for dynamic models.
\end{itemize}


\section{Approximate Bayesian computation} \label{sec:abc}

Approximate Bayesian computation (ABC) is an increasingly popular inference method for model parameters $\theta$, in that it only requires the ability to produce artificial data from a stochastic model \textit{simulator} \cite{beaumont2002approximate,marin2012approximate}.  A simulator is essentially a computer program, which takes $\theta$, makes internal calls to a random number generator, and outputs a vector of artificial data. The implication is that ABC can be used to produce approximate inference when the likelihood function $p(y|\theta)$ underlying the simulator is intractable. As such ABC methods have been applied to a wide range of disciplines \cite{sisson2018handbook}. The fundamental idea in ABC is to generate parameter proposals $\theta^{\star}$ and accept a proposal if the simulated data $y^{\star}$ for that proposal is similar to observed data $y^{\text{obs}}$. Typically this approach is not suitable for high-dimensional data, and a set of summary statistics of the data is therefore commonly introduced to break the \textit{curse-of-dimensionality}. So, instead of comparing $y^{\star}$  to $y^{\text{obs}}$, we compare summary statistics of the simulated data $s^{\star}=S(y^{\star})$ to those of observed data $s^{\text{obs}}=S(y^{\text{obs}})$. Then we accept the proposed $\theta^{\star}$ if $s^{\star}$ is close to $s^{\text{obs}}$ in some metric. Using this scheme, ABC will simulate draws from the following approximate posterior of $\theta$
\begin{equation*}
    p^{\epsilon}_{\text{ABC}}(\theta | s^{\text{obs}}) \propto \int K_\epsilon(\Delta(s^{\star},s^{\text{obs}}))p(s^{\star}|\theta)p(\theta)ds^{\star},
\end{equation*}
where $p(\theta)$ is the prior of $\theta$, $\Delta$ is a distance function between observed and simulated summaries (we use a Mahalanobis distance, see the supplementary material in Appendix \ref{sec:suppl_mat}), $K_\epsilon(\cdot)$ is a kernel, which in all our applications is the uniform kernel returning 1 if $\Delta(s^{\star},s^{\text{obs}})<\epsilon$ and 0 otherwise, and $\epsilon>0$ is the so-called ABC-threshold. A smaller $\epsilon$ produces more accurate approximations to the true summaries posterior $p(\theta | s^{\text{obs}})$, though this implies a larger computational effort due to the increasing number of rejected proposals. An additional issue is that ideally we would like to target $p(\theta | y^{\text{obs}})$, not $p(\theta | s^{\text{obs}})$, but again unless sufficient statistics are available (impossible outside the exponential family), and since $\epsilon>0$, we have to be content with samples from $p^{\epsilon}_{\text{ABC}}$.

In this work we do not focus on \textit{how} to sample from $p^{\epsilon}_{\text{ABC}}(\theta | s^{\text{obs}})$ (see \citealp{sisson2018handbook} for possibilities). Therefore, we employ the simplest (and also most inefficient) ABC algorithm, the so called ``ABC rejection sampling'' \cite{pritchard1999population}. We will use the ``reference table'' version of ABC rejection sampling (e.g.~\citealp{cornuet2008inferring}), which is as follows:
\begin{itemize}
    \item Generate $\tilde{N}$ independent proposals $\theta^i\sim p(\theta)$,  and corresponding data  $y^i\sim p(y|\theta^i)$ from the simulator;
    \item Compute the summary statistics $s^i = S(y^i)$ for each $i=1,...,\tilde{N}$;
    \item Compute the distances $\Delta(s^{i}, s^{\text{obs}})$ for each $i=1,...,\tilde{N}$;
    \item Retain proposals $\theta^i$ corresponding to those $\Delta(s^{i}, s^{\text{obs}})$ that are smaller than the $x$-th percentile of all distances.
\end{itemize}
The retained $\theta^i$'s form a sample from $p^{\epsilon}_{\text{ABC}}$ with $\epsilon$ given by the selected $x$th percentile.
An advantage of this approach is that it allows to easily compare the quality of the ABC inference based on several methods for computing the summaries, under the same computational budget $\tilde{N}$. Moreover, once the ``reference table'' $(\theta^i,y^i)_{1\leq i \leq \tilde{N}}$ has been produced in the first step, we can recycle these simulations to produce new posterior samples using several methods for computing the summary statistics. 

\subsection{Learning summary statistics}
Event though ABC rejection sampling is highly inefficient due to proposing parameters from the prior $p(\theta)$, this is not a concern for the purpose of our work.
In fact, our main focus is \textit{learning} the summary statistics $S(\cdot)$. This is perhaps the most serious difficulty affecting the application of ABC methodology to practical problems. In fact, we require summaries that are informative for $\theta$, as a replacement for the (unattainable) sufficient statistics. A considerable amount of research has been conducted on how to construct informative summary statistics (see \citealp{blum2013comparative} and \citealp{prangle2015summary} for an overview). However their selection is still challenging since no state-of-the-art methodology exists that can be applied to arbitrarily complex problems. \citet{fearnhead2012constructing} consider a regression-based approach where they also show that the best summary statistic, in terms of the minimal quadratic loss, is the posterior mean. The latter is however unknown since $p(\theta | y^{\text{obs}})$ itself is unknown. Therefore, they introduce a simulation approach based on a linear regression model
\begin{equation}
\theta_j^i = E(\theta_j | y^i) + \xi_j^i = b_{0_j} + b_j h(y^i) + \xi_j^i
\label{eq:linera_regerssion_model_fearnhead2012constructing}
\end{equation}
with $\xi_j^i$ some mean-zero noise. Here $j=1,...,\dim(\theta)$ and $h(y^i)$ is a vector of (non)-linear transformations of ``data'' $y^i$ (here $y^i$ can be simulated or observed data). Therefore \citet{fearnhead2012constructing} have $\dim(\theta)$ models to fit separately, one for each component of vector $\theta$. Of course, these fittings are to be performed \textit{before} ABC rejection is executed, so this is a step that anticipates ABC rejection, to provide the latter with suitable summary statistics. The parameters in each regression \eqref{eq:linera_regerssion_model_fearnhead2012constructing} are estimated by fitting the model by least squares to a new set of $N$ simulated data-parameter pairs ${(\theta^{i}, y^i)}_{1 \leq i \leq N}$ where, same as for ABC rejection, the $\theta^i$ are generated from $p(\theta)$ and the $y^i$ are generated from the model simulator conditionally on $\theta^i$. To clarify the notation: $N$ is the number of data-parameter pairs used to fit the linear regression model in \eqref{eq:linera_regerssion_model_fearnhead2012constructing}, while $\tilde{N}$ is the number of parameter-data pair proposals used in ABC rejection sampling. However the two sets of parameter-data pairs ${(\theta^{i}, y^i)}_{1 \leq i \leq N}$ and ${(\theta^{i}, y^i)}_{1 \leq i \leq \tilde{N}}$ are different since these serve two separate purposes. They are generated in the same way but independently of each other.  After fitting \eqref{eq:linera_regerssion_model_fearnhead2012constructing}, estimates $(\hat{b}_{0_j},\hat{b}_j)$ are returned and $\hat{b}_{0_j} + \hat{b}_j h(y)$ is taken as $j$th summary statistic, $j=1,...,\dim(\theta)$. We can then take $S_j(y^{\text{obs}}) =  \hat{b}_{0_j} + \hat{b}_j h(y^{\text{obs}})$ as $j$th component of $S(y^{\text{obs}})$, and similarly take $S_j(y^{\star}) = \hat{b}_{0_j} + \hat{b}_j h(y^{\star})$. The number of summaries is therefore equal to the size of $\theta$.

This approach is further developed in \citet{jiang2017learning} where a MLP  deep neural network regression model is employed, and replaces the linear regression model in \eqref{eq:linera_regerssion_model_fearnhead2012constructing}. Hence, \citet{jiang2017learning} has the following regression model
\begin{equation*}
\theta^i = E(\theta|y^i) + \xi^i = f_{\beta}(y^i) + \xi^i
\label{eq:linera_regerssion_model_dnn}
\end{equation*}
where $f_{\beta}$ is the MLP parametrized by the weights $\beta$. \citet{jiang2017learning} estimate $\beta$ from
\begin{equation} \label{eq:minimization_problem_dnn}
\min_{\beta}  \frac{1}{N} \sum_{i=1}^{N} \| f_{\beta}(y^i) - \theta^{i} \|^{2}_{2},
\end{equation}
where ${(\theta^{i}, y^i)}_{1\leq i \leq N}$ are the parameter-data pairs that the network $f_{\beta}$ is fitted to.
\ifx
After fitting the regression model we have that $S(y) = f_{\hat{\beta}}(y)$ is the summary statistic. The full scheme for running ABC in \citet{jiang2017learning} is as follows:

\begin{itemize}
    \item Generate a large number $N$ of parameter-data pairs ${(\theta^{i}, y^i)}_{1 \leq i \leq N}$, by sampling $\theta^{i}$ from the prior, and then simulate corresponding data set $y^{i}$ from the simulator;
    \item Train $f_{\beta}$ on the generated set of parameters-data;
    \item Run ABC where the network $f_{\hat{\beta}}$ is used to provide summary statistics.
\end{itemize}
\fi

The deep neuronal network with  multiple  hidden  layers considered in \citet{jiang2017learning}  offers  stronger  representational  power to approximate $E(\theta|y)$ (and hence learn an informative summary statistic),  compared  to using linear regression, if the posterior mean is a highly non-linear function of $y$. Moreover, experiments in \citet{jiang2017learning} show that indeed their MLP outperforms the linear regression approach in \citet{fearnhead2012constructing} (at least for their considered experiments), although at the price of a much larger computational effort. For this reason in our experiments we compare ABC coupled with PENs with the ABC MLP from \citet{jiang2017learning}.

In \citet{creel2017neural} a deep neural network regression model is used. He also introduces a pre-processing step such that instead of feeding the network with the data set $y^{\mathrm{obs}}$, the network is fed with a set of statistics of the data $s^{\mathrm{obs}}$. This means that, unlike in \citet{jiang2017learning}, in \citet{creel2017neural} the statistician must already know ``some kind'' of initial summary statistics, used as input, and then the network returns another set of summary statistics as output, and the latter are used for ABC inference. Our PENs do not require any initial specification of summary statistics.

\section{Partially exchangeable networks} \label{sec:pens}



Even though the likelihood function is intractable in the likelihood-free setting, we may still have insights into properties of the data generating process. To that end, given our data set $y \in \mathcal{Y}^M$ with $M$ units, we will exploit some of the invariance properties of its prior predictive distribution $p(y)=\int_\theta p(y|\theta)p(\theta)d\theta$. As discussed in \cref{sec:abc}, the regression approach to ABC \citep{fearnhead2012constructing} involves to learn the regression function $y \mapsto E(\theta | y)$, where $E(\theta | y)$ is the posterior mean. Our goal in this section is to leverage the invariances of the Bayesian model $p(y)$ to design deep neural architectures that are fit for this purpose.

\subsection{Exchangeability and partial exchangeability}

The simplest form of model invariance is \emph{exchangeability}. A model $p(y)$ is said to be exchangeable if, for all permutations $\sigma$ in the symmetric group $S_M$, $p(y)=p(y_{\sigma(1)},...,y_{\sigma(M)})$. For example, if the observations are independent and identically distributed (i.i.d.)~given the parameter, then $p(y)$ is exchangeable. A famous theorem of \citet{definetti1929}, which was subsequently generalized in various ways (see e.g.~the review of \citealp{diaconis1988}), remarkably shows that such conditionally i.i.d.~models are essentially the only exchangeable models.

If the model is exchangeable, it is clear that the function $y \mapsto E(\theta | y)$ is permutation invariant. It is therefore desirable that a neural network used to approximate this function should also be permutation invariant. The design of permutation invariant neural architectures has been the subject of numerous works, dating at least back to \citet[Chap. 2]{minsky1988perceptrons} and \citet{shawe1989}. A renewed interest in such architectures came about recently, notably through the works of \citet{ravanbakhsh2017deep}, \citet{zaheer2017deep}, and \citet{murphy2018janossy}---a detailed overview of this rich line of work can be found in \citet{bloem2019}. Most relevant to our work is the DeepSets architecture of \citet{zaheer2017deep} that we generalize to partial exchangeability, and the approach of \citet{chan2018likelihood}, who used permutation invariant networks for ABC.

However, the models considered in ABC are arising from intractable-likelihoods scenarios, which certainly are not limited to exchangeable data, quite the opposite, e.g.~stochastic differential equations \cite{picchini2014inference}, state-space models and beyond \cite{jasra2015approximate}. To tackle this limitation, we ask: \emph{could we use a weaker notion of invariance to propose deep architectures suitable for such models?} In this paper, we answer this question for a specific class of non-i.i.d.~models: Markov chains. To this end, we make use of the notion of \emph{partial exchangeability} studied by \citet{diaconis1980}. This property can be seen as a weakened version of exchangeability where $p(y)$ is only invariant to a subset of the symmetric group called \emph{block-switch transformations}. Informally, for $d \in \mathbb{N}$, a $d$-block-switch transformation interchanges two given disjoint blocks of $y \in \mathcal{Y}^M$ when these two blocks start with the same $d$ symbols and end with the same $d$ symbols.
\begin{definition}[\textbf{Block-switch transformation}]
For increasing indices $b=(i,j,k,l) \in \{0,\ldots,M \}^4$ such that $j-i\geq d$ and $l-k\geq d$, the \textbf{$d$-block-switch transformation} $T^{(d)}_b$ is defined as follows: if $y_{i:(i+d)}=y_{k:(k+d)}$ and $y_{(j-d):j}=y_{(l-d):l}$ then
\begin{align}
\everymath{\displaystyle}
y &= y_{1:i-1}
\tikz[baseline]{\node[fill=red!20,anchor=base] (b1ij) {$y_{i:j}$};}
y_{(j+1):(k-1)}
\tikz[baseline]{\node[fill=blue!20,anchor=base] (b1kl) {$y_{k:l}$};}
y_{(l+1):M} \\
T^{(d)}_b(y) &= y_{1:i-1}
\tikz[baseline]{\node[fill=blue!20,anchor=base] (b2kl) {$y_{k:l}$};}
y_{(j+1):(k-1)}
\tikz[baseline]{\node[fill=red!20,anchor=base] (b2ij) {$y_{i:j}$};}
y_{(l+1):M}.
\begin{tikzpicture}[overlay]
        \path[->,draw=red,shorten >=7.5pt,shorten <=7.5pt] (b1ij) edge  (b2ij);
        \path[->,draw=blue,shorten >=7.5pt,shorten <=7.5pt] (b1kl) edge  (b2kl);
\end{tikzpicture}%
\end{align}
If $y_{i:(i+d)} \neq y_{k:(k+d)}$ or $y_{(j-d):j} \neq y_{(l-d):l}$ then the block-switch transformation leaves $y$ unchanged: $T^{(d)}_b(y)=y$.
\end{definition}
\begin{definition}[\textbf{Partial exchangeability}] Let $A$ be a metric space.
A function $F: \mathcal{Y}^M \rightarrow A$ is said to be \textbf{$d$-block-switch invariant} if $F(y) = F(T_b^{(d)}(y))$ for all $y \in \mathcal{Y}$ and for all $d$-block-switch transformations $T_b^{(d)}$. Similarly, a model $p(y)$ is \textbf{$d$-partially exchangeable} if for all $d$-block-switch transformations $T_b^{(d)}$ we have $p(y)=p(T_b^{(d)}(y))$.
\end{definition}
Note that $0$-partial exchangeability reduces to exchangeability and that all permutations are $0$-block-switch transformations.

It is rather easy to see that, if $p(y|\theta)$ is a Markov chain of order $d$, then $p(y)$ is partially exchangeable (and therefore $y \mapsto E(\theta | y)$ is $d$-block-switch invariant). In the limit of infinite data sets, \citet{diaconis1980} showed that the converse was also true: any partially exchangeable distribution is conditionally Markovian. This result, which is an analogue of de Finetti's theorem for Markov chains, justifies that \emph{partial exchangeability is the right symmetry to invoke when dealing with Markov models}.

\subsection{From model invariance to network architecture}

When dealing with Markovian data, we therefore wish to model a regression function $y \mapsto E(\theta | y)$ that is $d$-block-switch invariant. Next theorem gives a general functional representation of such functions, in the case where $\mathcal{Y}$ is countable.

\begin{theorem}\label{th:decomposition}
Let $F: \mathcal{Y}^M \rightarrow A$ be $d$-block-switch invariant. If $\mathcal{Y}$ is countable, then there exist two functions  $\phi :\mathcal{Y}^{d+1}   \rightarrow  \mathbb{R}$ and $\rho :\mathcal{Y}^d \times \mathbb{R} \rightarrow  A$ such that
\begin{equation}
    \forall y \in \mathcal{Y}^M, \; F(y) = \rho \left(y_{1:d}, \sum_{i=1}^{M-d} \phi \left(y_{i:(i+d)} \right) \right).
    \label{eq:decomposition}
\end{equation}
\end{theorem}
\begin{proof}
Let $\sim$ be the equivalence relation over $\mathcal{Y}^M$ defined by $$x\sim y \iff \exists b_1,\ldots,b_k, \; \;   y = T_{b_1}^{(d)} \circ \cdots \circ T_{b_k}^{(d)} (x).$$ 
Let $\text{cl}:\mathcal{Y}^M \rightarrow \mathcal{Y}^M/{\sim}$ be the projection over the quotient set. According to the properties of the quotient set, since $F$ is $d$-block-switch invariant, there exists a unique function $g: \mathcal{Y}^M/{\sim} \rightarrow A$ such that $F = g \circ \text{cl}$.

Since $\mathcal{Y}$ is countable, $\mathcal{Y}^{d+1}$ is also countable and there exists an injective function $c: \mathcal{Y}^{d+1} \rightarrow \mathbb{N}$. Consider then the function
$$\nu: y \mapsto \left(y_{1:d},\sum_{i=1}^{M-d} 2^{-c(y_{i:(i+d)})} \right),$$
which is clearly $d$-block-switch invariant. There exists a unique function $h: \mathcal{Y}^M/{\sim} \rightarrow \nu(\mathcal{Y}^M)$ such that $\nu =  h \circ \text{cl}.$

We will now show that $h$ is a bijection. By construction, $h$ is clearly surjective. Let us now prove its injectivity. We thus have to show that, for all $x,y \in \mathcal{Y}^M$, $\nu(x)=\nu(y)$ implies $x \sim y$. Let $x,y \in \mathcal{Y}^M$ such that $\nu(x)=\nu(y)$. We have therefore $x_{1:d} = y_{1:d}$ and $$\sum_{i=1}^{M-d} 2^{-c(x_{i:(i+d)})}= \sum_{i=1}^{M-d} 2^{-c(y_{i:(i+d)})}.$$
The uniqueness of finite binary representations then implies that $\{x_{i:(i+d)}\}_{i\leq M-d}=\{y_{i:(i+d)}\}_{i\leq M-d}$. According to \citet[Proposition 27]{diaconis1980}, those two conditions imply that $x \sim y$, which shows that $h$ is indeed injective.

Since $h$ is a bijection, $\nu =  h \circ \text{cl}$ implies that $\text{cl} = h^{-1} \circ \nu$ which leads to $F = g \circ h^{-1} \circ \nu$. Finally, expanding this gives
$$ \forall y \in \mathcal{Y}^M, \; F(y) = g \circ h^{-1} \left(y_{1:d}, \sum_{i=1}^{M-d} 2^{-c(y_{i:(i+d)})}  \right),$$
which is the desired form with $\phi(y) = 2^{-c(y)}$ and $\rho = g \circ h^{-1}$.
\end{proof}

When $d=0$, the representation reduces to
\begin{equation}
     F(y) = \rho \left( \sum_{i=1}^{M} \phi \left(y_{i} \right) \right),
\end{equation}
and we exactly recover Theorem 2 from \citet{zaheer2017deep}---which also assumes countability of $\mathcal{Y}$---and the DeepSets representation. While an extension of our theorem to the uncountable case is not straightforward, we conjecture that a similar result holds even with uncountable $\mathcal{Y}$. A possible way to approach this conjecture is to study the very recent and fairly general result of \citet{bloem2019}. We note that the experiments on an autoregressive time series model in \cref{sec:AR2}, which is a Markovian process, support this conjecture.

\paragraph{Partially exchangeable networks}
The result in \cref{th:decomposition} suggests how to build $d$-block-switch invariant neural networks: we replace the functions $\rho$ and $\phi$ in \cref{eq:decomposition} by feed forward neural networks and denote this construction a $d$-partially exchangeable network (PEN-$d$ or PEN of order $d$). In this construction, we will call $\phi$ the \emph{inner network}, which maps a $d$-length subsequence $y_{i:i+d}$ into some representation $\phi(y_{i:i+d})$, and $\rho$ is the \emph{outer network} that maps the first $d$ symbols of the input, and the sum of the representations of all $d$-length subsequences of the input, to the output. We note that DeepSets networks are a special case of the PENs that corresponds to PEN-0.

\subsection{Using partially exchangeable networks for learning summary statistics for ABC}\label{sec:PENABC}


While PENs can by used for any exchangeable data, in this paper we use it for learning summary statistics in ABC. In particular, we propose the following regression model for learning the posterior mean
\begin{equation*}
       \theta^i = E(\theta|y^i) + \xi^i = \rho_{\beta_{\rho}}\biggl(y^i_{1:d}, \sum_{l = 1}^{M-d}  \phi_{\beta_{\phi}}(y^i_{l:l+d})\biggr) + \xi^i.
\end{equation*}
Here $\beta_{\phi}$ are the weights for the inner network, and $\beta_{\rho}$ are the weights for the outer network that maps its arguments into the posterior mean of the unknown parameters, which is the ABC summary we seek. When using PENs to learn the summary statistics we obtain the weights for the networks using the same criterion as in \cref{eq:minimization_problem_dnn}, except that instead of using the MLP network we use a PEN network for the underlying regression problem.

When targeting static models we employ a PEN-0, i.e.~a DeepSets network, since a static model can be viewed as a zero-order Markov model. For time series models we use a PEN-$d$, where $d>0$ is the order of the assumed data generating Markov process.

\section{Experiments} \label{sec:experiments}

We present four experiments: two static models (g-and-k and $\alpha$-stable distributions), and two time series models (autoregressive and moving average models). Full specification of the experimental settings is provided as supplementary material. The code was written in Julia 1.0.0 \cite{bezanson2017julia} and the framework Knet \cite{yuret2016knet} was used to build the deep learning models. The code can be found at \href{https://github.com/SamuelWiqvist/PENs-and-ABC}{https://github.com/SamuelWiqvist/PENs-and-ABC}. All experiments are simulation studies and the data used can be generated from the provided code. We compare approximate posteriors to the true posteriors using the Wasserstein distance, which we compute via the POT package \citep{flamary2017pot}. This distance can be sensitive to the number of posterior samples used, however, we observed that our results are fairly robust to variations in the number of samples. In all experiments we used 100 posterior samples to estimate the Wasserstein distance, except for the AR2 model where we used 500 samples.
We also employ two different MLP networks: ``MLP small'', where we use approximately the same number of weights as for the PEN-$d$ network; and ``MLP large'', which has a larger number of weights than PEN-$d$.

\subsection{g-and-k distribution} \label{sec:gandk}

The g-and-k distribution is defined by its quantile function via four parameters, and not by its probability density function since the latter is unavailable in closed form. This means that the likelihood function is ``intractable'' and as such exact inference is not possible. However, it is very simple to simulate draws from said distribution (see the supplementary material in Appendix \ref{sec:suppl_mat}), which means that g-and-k models are often used to test ABC algorithms \cite{prangle2017gk}.

The unknown parameters are $\theta = [A,B,g,k]$ (for full specification of the g-and-k distribution, see the supplementary material in Appendix \ref{sec:suppl_mat}). The prior distributions are set to $p(A) \sim \Gamma(2,1) $, $p(B) \sim \Gamma(2,1)$, $p(g) \sim \Gamma(2,0.5)$, and $p(k) \sim \Gamma(2,1)$ ($\Gamma(\alpha, \beta)$ is the Gamma distribution with shape parameter $\alpha$ and rate parameter $\beta$). We perform a simulation study with ground-truth parameters $A = 3$, $B = 1$, $g = 2$, $k = 0.5$ (same ground-truth parameter values as in \citealp{allingham2009bayesian}, \citealp{picchini2017approximate}, \citealp{fearnhead2012constructing}). Our data set comprises $M=1,000$ realizations from a g-and-k distribution.

We compare five different methods of constructing the summary statistics for ABC: (i) the handpicked summary statistics in \citet{picchini2017approximate}, i.e. $S(y) = [P_{20}, P_{40}, P_{60}, P_{80}, \text{skew}(y)]$ ($P_{i}$ is the $i$th percentile and $\text{skew}(y)$ is the skewness); (ii) ``MLP small''; (iii) ``MLP large''; (iv) a MLP network with a preprocessing step, denoted ``MLP pre'', where we feed the network with the empirical distribution function of the data instead of feeding it with the actual data; and (v) PEN-0 (DeepSets) since the data is i.i.d. the order of the Markov model is 0).

The probability density function for the g-and-k distribution can be approximated via finite differences, as implemented in the \texttt{gk} R package \cite{prangle2017gk}. This allow us to sample from an almost exact posterior distribution using standard Markov chain Monte Carlo (MCMC). We evaluate the inference produced using summaries constructed from the five methods (i--v) by comparing the resulting ABC posteriors to the ``almost exact'' posterior (computed using MCMC). ABC inferences are repeated over 100 independent data sets, and for a different number of training data observations for DNN models.
The results are presented in \cref{fig:res_gandk} and we can conclude that PEN-0 generates the best results. Furthermore, PEN-0 is also more data efficient since it performs considerably better than other methods with limited number of training observations. It seems in fact that PEN-0 requires 10 times less training data than ``MLP pre'' to achieve the same inference accuracy.
However all methods performed poorly when too few training observations are used.
The results also show that when MLP is fed with the observations it generates poor results, but if we instead use ``MLP pre'' and send in the empirical distribution function, in the spirit of \citet{creel2017neural}, we obtain considerably better results.

\begin{figure}[ht]
\vskip 0.2in
\begin{center}
\centerline{\includegraphics[width=1\columnwidth]{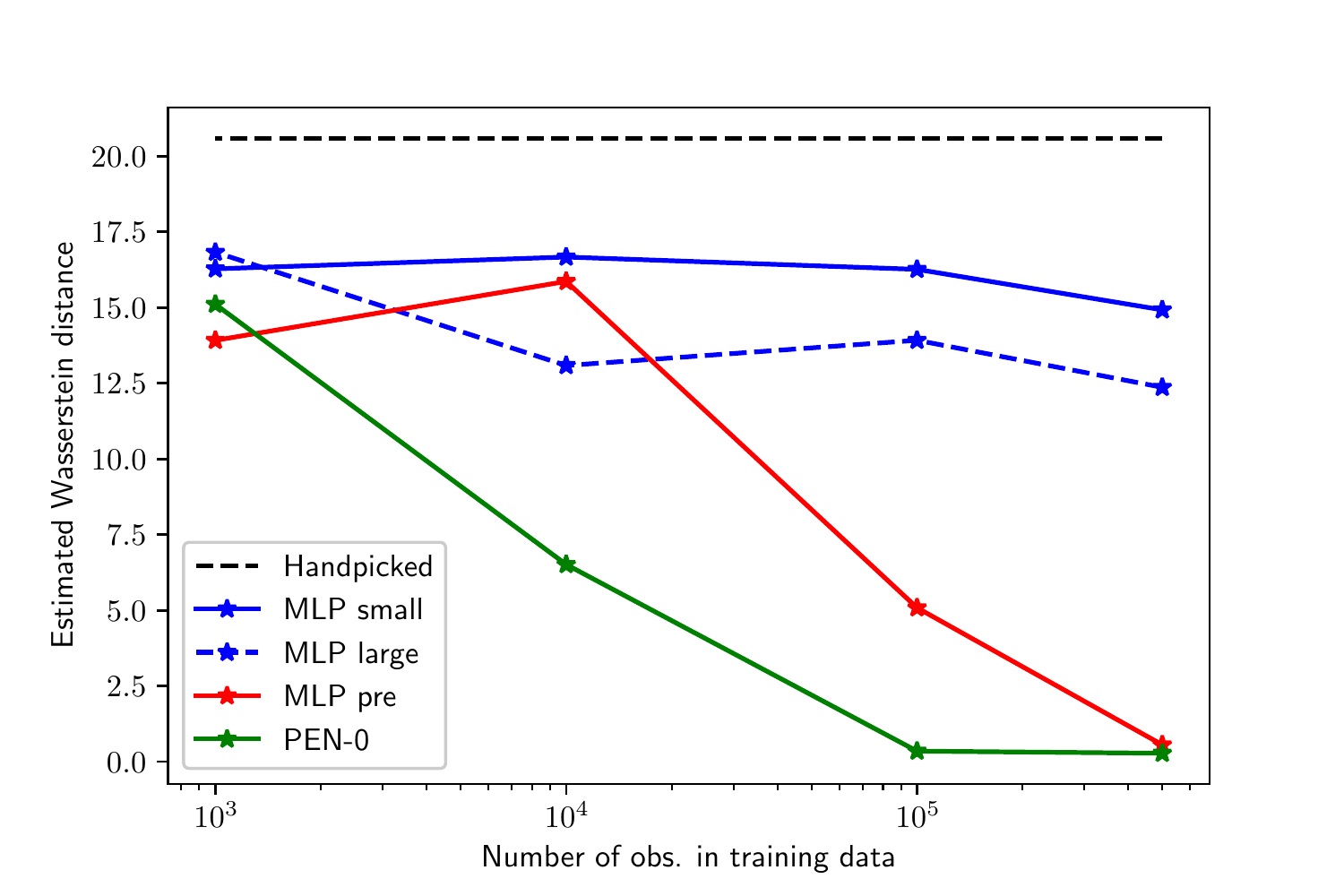}}
\caption{Results for g-and-k distribution: The estimated Wasserstein distances (mean over 100 repetitions) when comparing the MCMC posterior with ABC posteriors.}
\label{fig:res_gandk}
\end{center}
\vskip -0.2in
\end{figure}

\subsection{$\alpha$-stable distribution} \label{sec:alphstable}

The $\alpha$-stable is a heavy-tailed distribution defined by its characteristic function (see supplementary material in Appendix \ref{sec:suppl_mat}). Its probability density function is intractable and inference is therefore challenging. Bayesian methods for the parameters can be found in e.g. \citet{peters2012likelihood} and \citet{ong2018variational}. Unknown parameters are $\theta = [\alpha, \beta, \gamma, \delta]$. We follow \citet{ong2018variational} and transform the parameters:
\begin{equation*}
    \tilde{\alpha} = \log \frac{\alpha-1.1}{2-\alpha}, \ \tilde{\beta} = \log \frac{\beta+1}{1-\beta}, \ \tilde{\gamma} = \log \gamma, \text{and} \ \tilde{\delta} = \delta.
\end{equation*}
This constraints the original parameters to $\alpha \in [1.1,2]$, $\beta \in [-1,1]$, and $\gamma > 0$. Independent Gaussian priors and ground-truth parameters  are as in \citet{ong2018variational}: $\tilde{\alpha}, \tilde{\beta}, \tilde{\gamma}, \tilde{\delta} \sim N(0,1)$; ground-truth values for the untransformed parameters are: $\alpha = 1.5$, $\beta = 0.5$, $\gamma = 1$, and $\delta = 0$. Observations consist of $M=1,000$ samples.

We compare methods for computing summary statistics as we did in \cref{sec:gandk} for the g-and-k distribution. However, since here the true posterior distribution is unavailable, we evaluate the different methods by comparing the root-mean square error (RMSE) between ground-truth parameter values and the ABC posterior means, see \cref{tab:alphastablermse}. From \cref{tab:alphastablermse} we conclude that PEN-0 performs best in terms of RMSE. Similarly to the g-and-k example we also see that ``MLP pre'' (see \cref{sec:gandk} for details) performs considerably better than MLP. 
We now look at the resulting posteriors. In \cref{fig:alphastableposteriors} five posteriors from five independent experiments are presented (here we have used $5 \cdot 10^5$ training data observations). Inference results when using handpicked summary statistics are poor and for $\tilde{\gamma}$ the posterior resembles the prior. Posterior inference is worst for ``MLP large''. Results for ``MLP pre'' and PEN-0 are similar, at least in the case depicted in \cref{fig:alphastableposteriors} where we use $5 \cdot 10^5$ training data observations. However, in terms of RMSE, PEN-0 returns the best results when we reduce the number of training data observations. 

\begin{table}[ht]
\caption{Results for $\alpha$-stable distribution. Root-mean square error (RMSE) when comparing posterior means to the ground-truth parameters (over 25 repetitions), for different methods of computing the summary statistics, and different number of training observations (between brackets).}
\label{tab:alphastablermse}
\vskip 0.15in
\begin{center}
\begin{small}
\begin{sc}
\resizebox{\columnwidth}{!}{
\begin{tabular}{lcccccc}
\toprule
{} &  Handpicked & MLP (small) & MLP (large) & MLP pre &  PEN-0  \\
\midrule
RMSE ($5\cdot10^5$) & 0.64   & 0.18 & 0.15 & 0.07 & 0.05 \\
RMSE ($10^5$)  &  0.64 & 0.19 & 0.17 & 0.07 & 0.06 \\
RMSE ($10^4$)  &  0.64 & 0.21 & 0.37 & 0.07 & 0.06\\
RMSE ($10^3$) & 0.64 & 0.72 & 0.62 & 0.40 & 0.07 \\
\bottomrule
\end{tabular}
}
\end{sc}
\end{small}
\end{center}
\vskip -0.1in
\end{table}

\begin{figure}[ht]
\vskip 0.2in
\centering

\subfigure[$\tilde{\alpha}$ (Handpicked) ]{\includegraphics[width=.18\columnwidth]{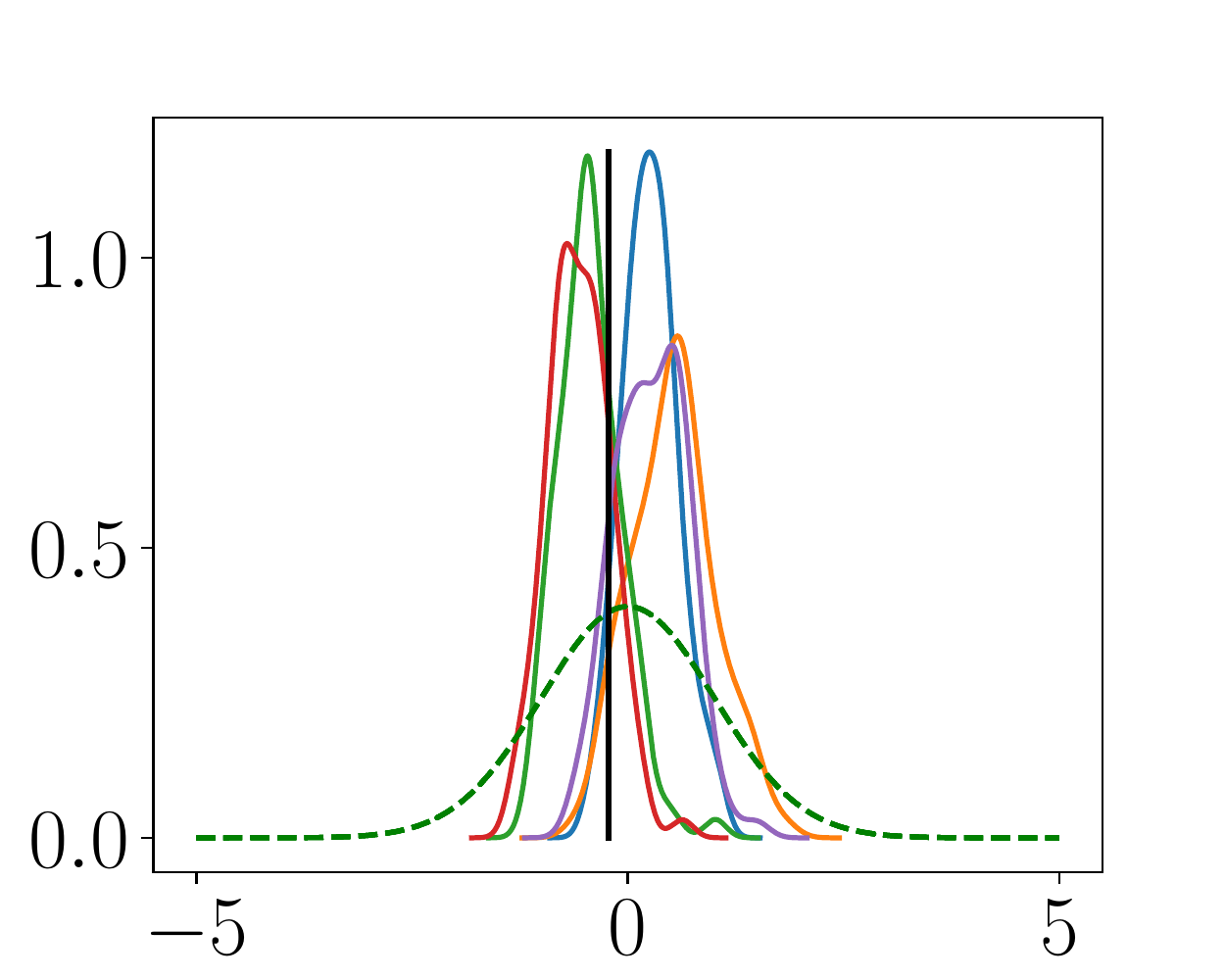}}\qquad
\subfigure[$\tilde{\alpha}$ (MLP large)]{\includegraphics[width=.18\columnwidth]{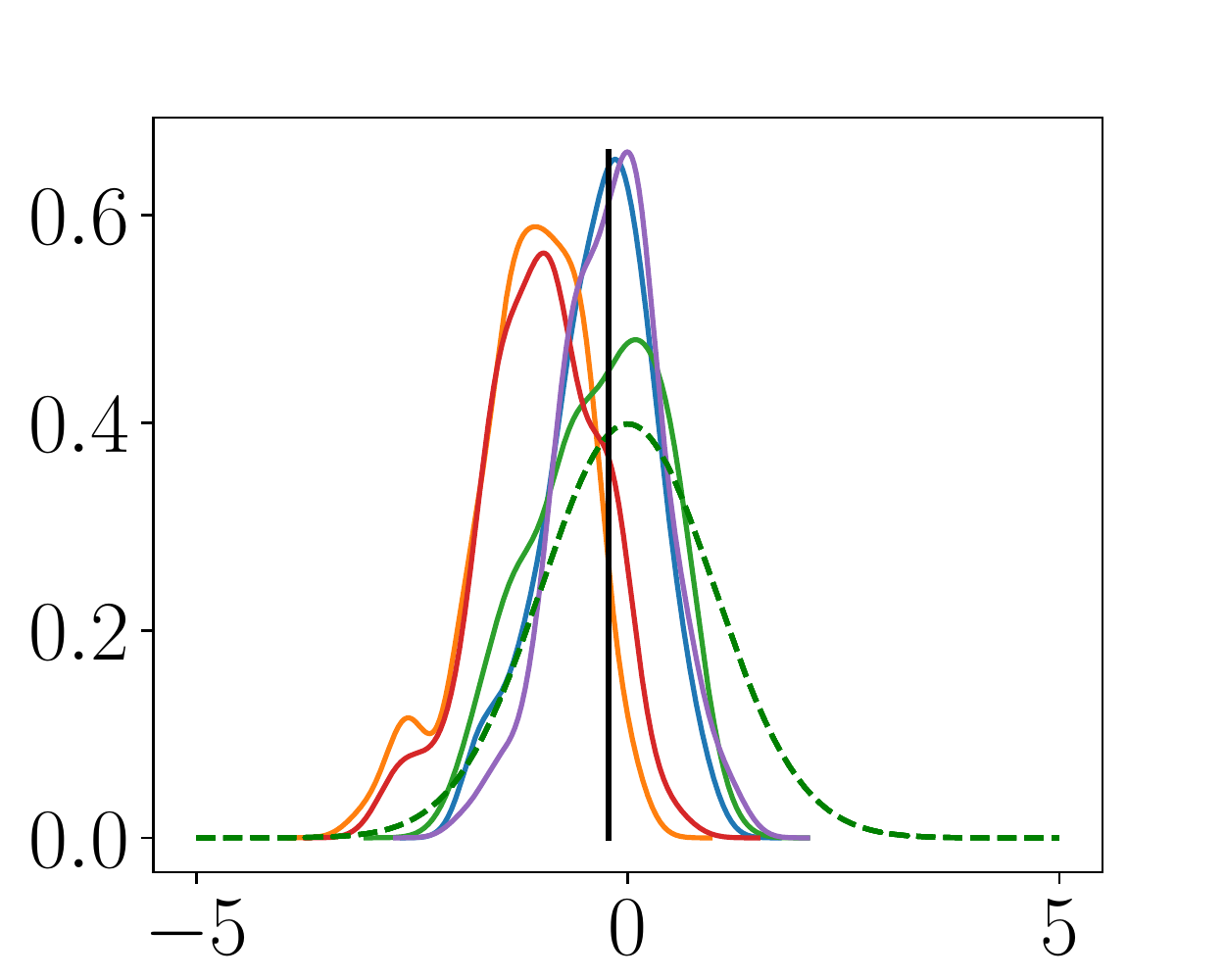}}\qquad
\subfigure[$\tilde{\alpha}$ (MLP pre)]{\includegraphics[width=.18\columnwidth]{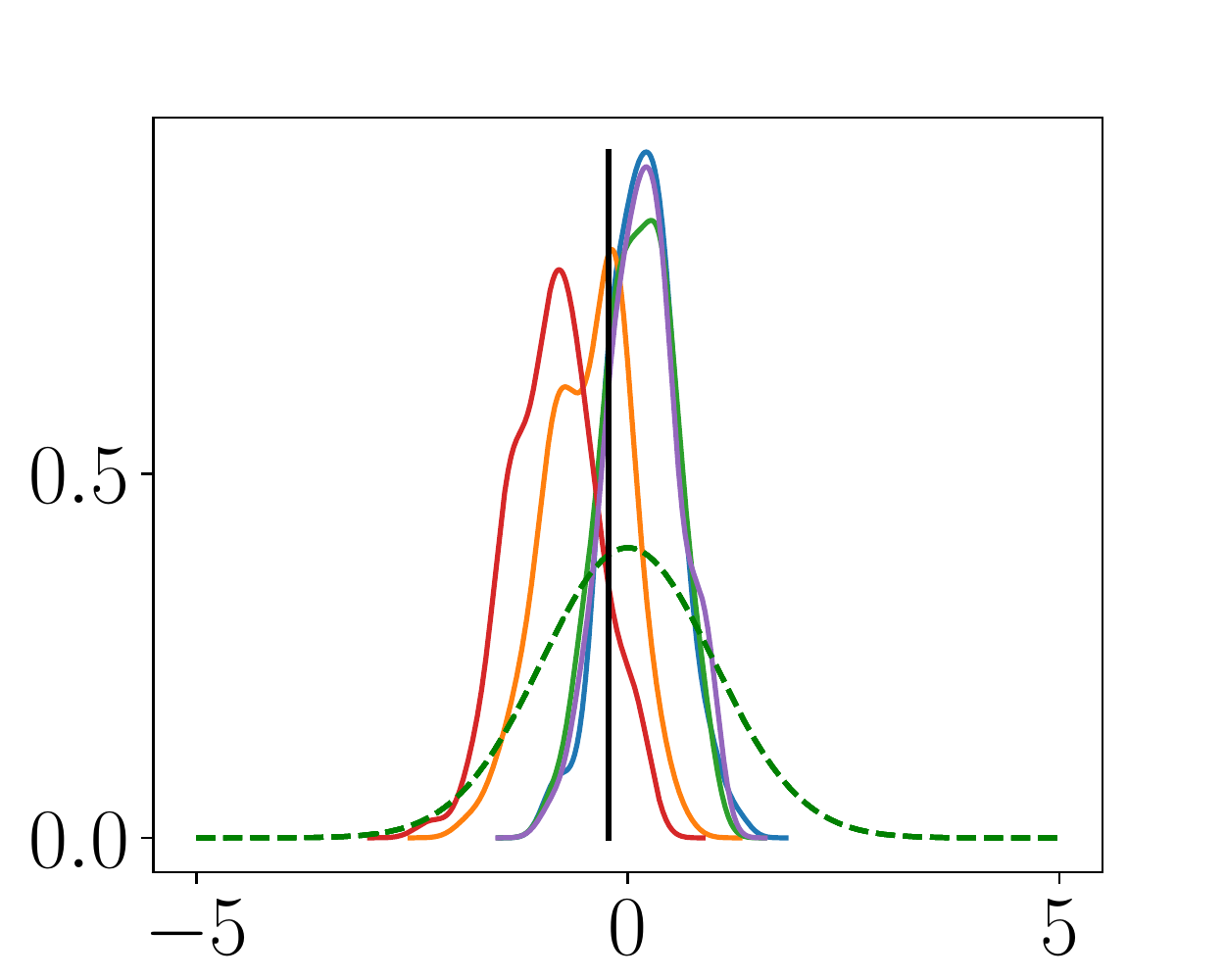}}\qquad
\subfigure[$\tilde{\alpha}$ (PEN-0)]{\includegraphics[width=.18\columnwidth]{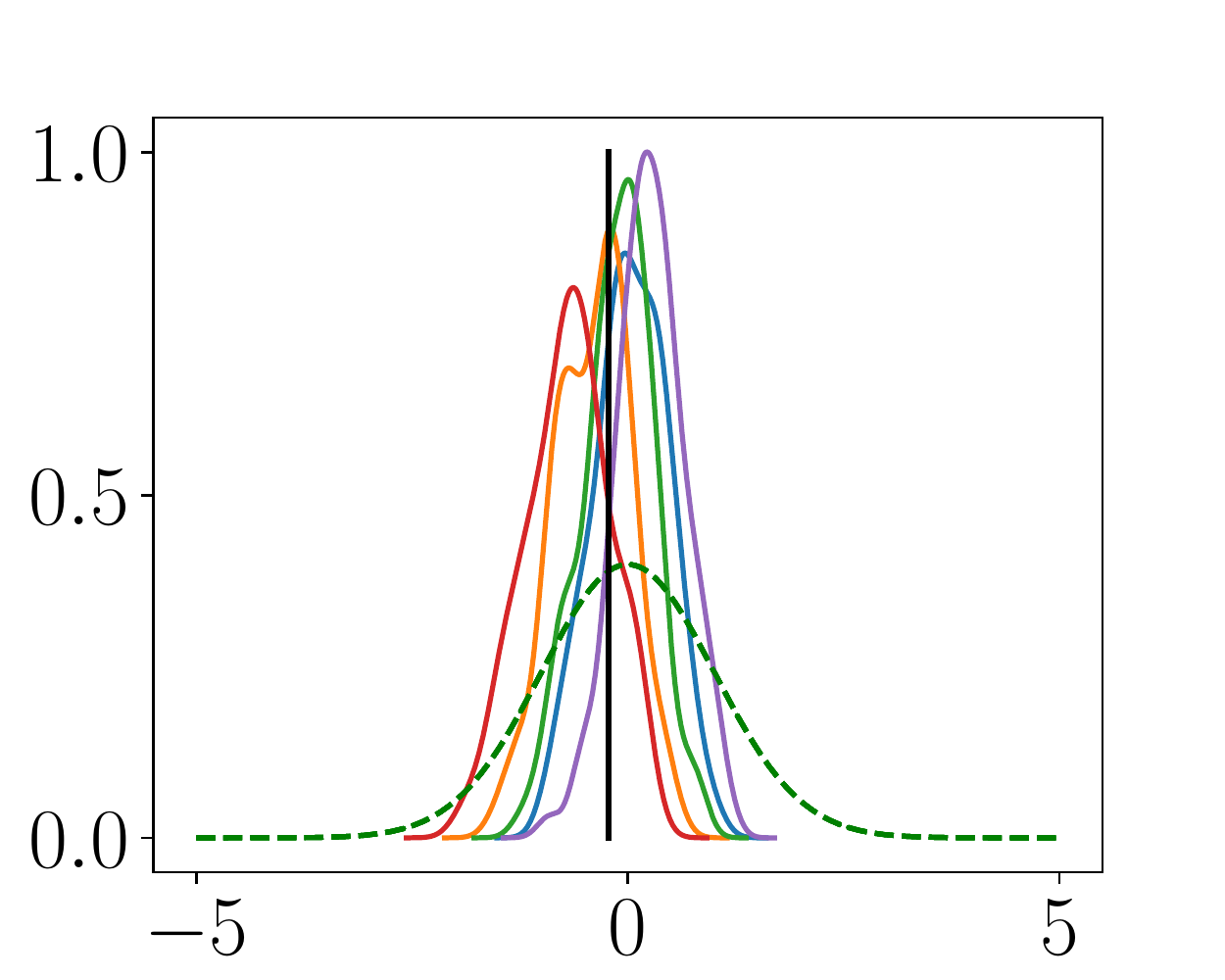}}
\\
\subfigure[$\tilde{\beta}$ (Handpicked) ]{\includegraphics[width=.18\columnwidth]{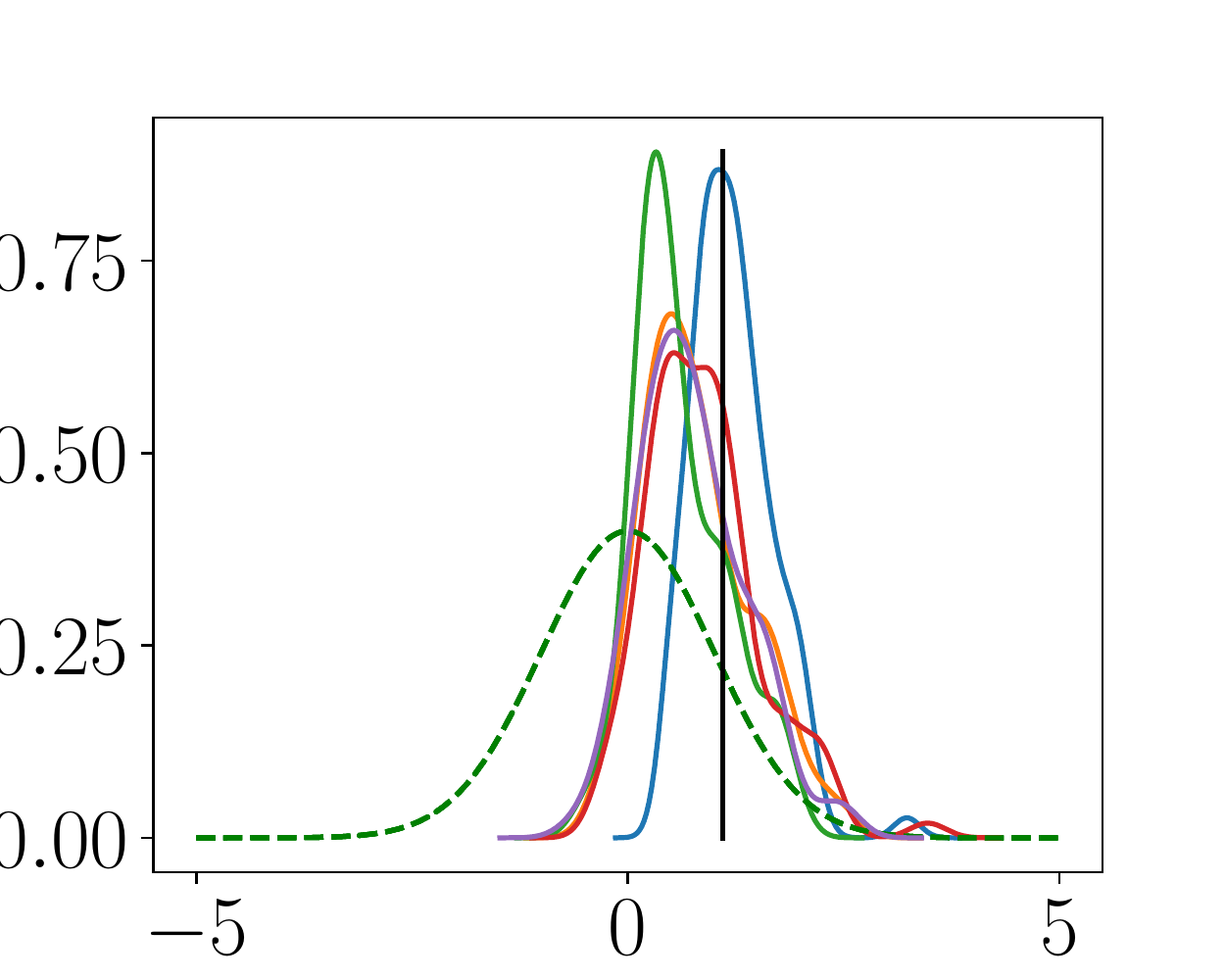}}\qquad
\subfigure[$\tilde{\beta}$ (MLP large)]{\includegraphics[width=.18\columnwidth]{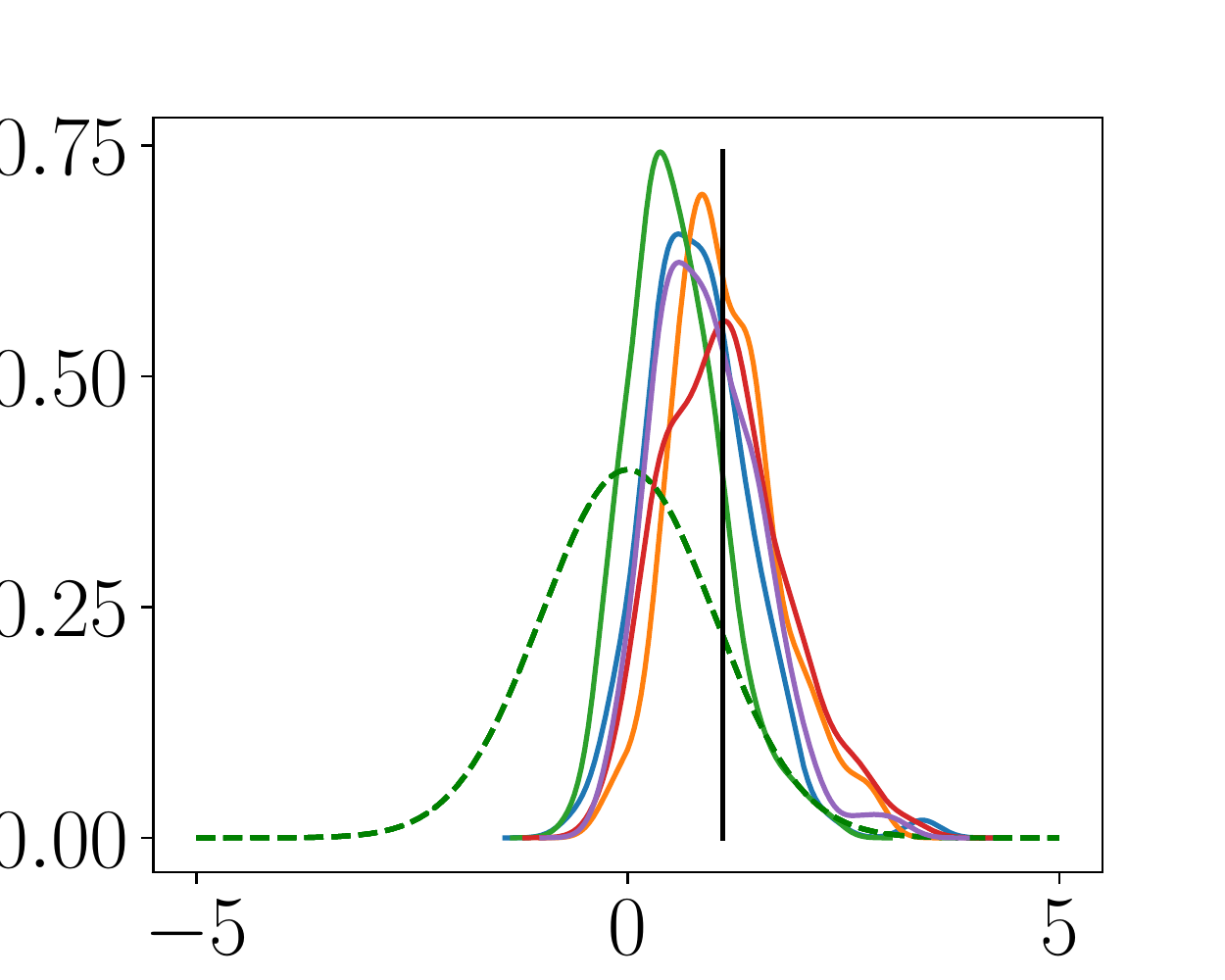}}\qquad
\subfigure[$\tilde{\beta}$ (MLP pre)]{\includegraphics[width=.18\columnwidth]{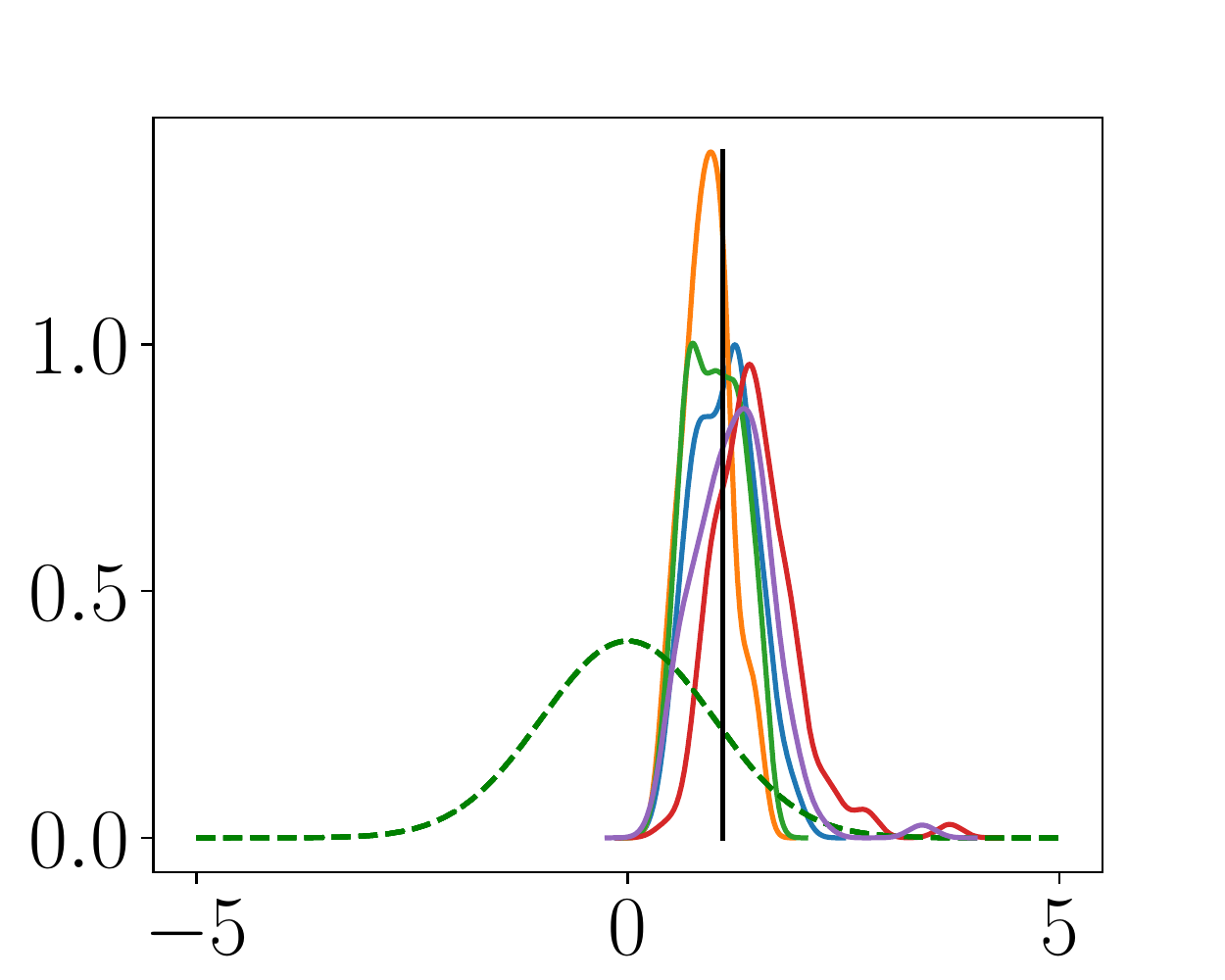}}\qquad
\subfigure[$\tilde{\beta}$ (PEN-0)]{\includegraphics[width=.18\columnwidth]{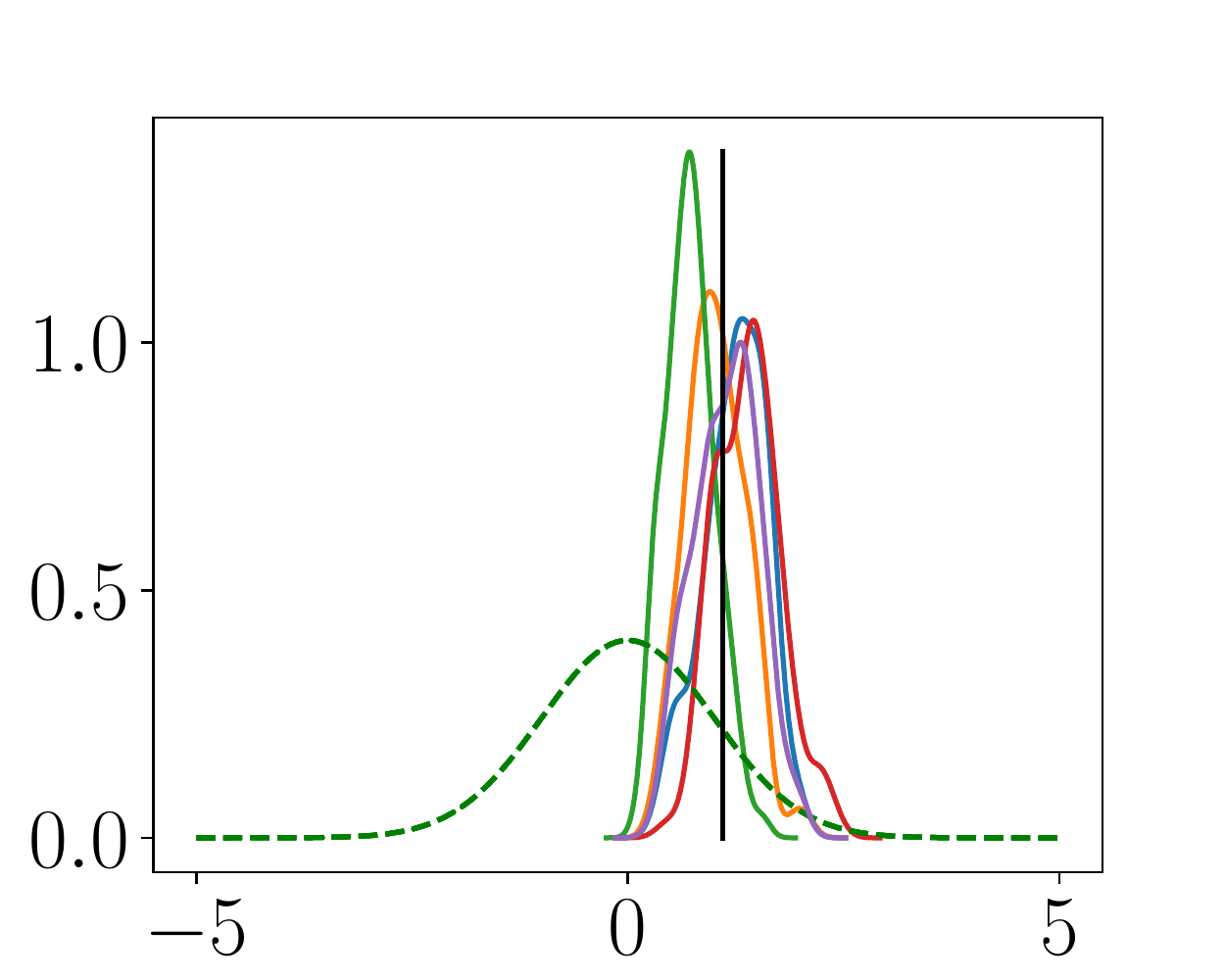}}
\\
\subfigure[$\tilde{\gamma}$ (Handpicked) ]{\includegraphics[width=.18\columnwidth]{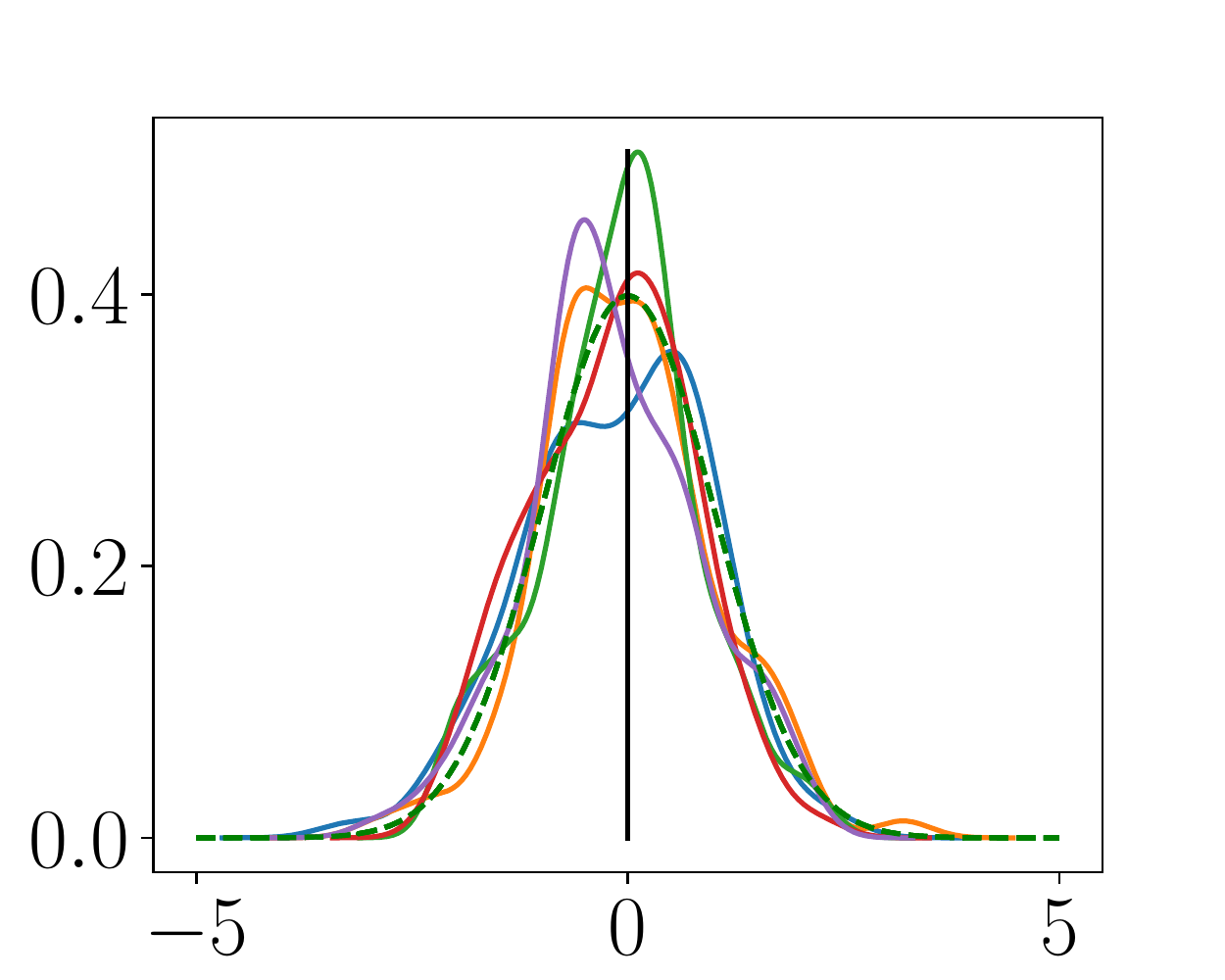}}\qquad
\subfigure[$\tilde{\gamma}$ (MLP large)]{\includegraphics[width=.18\columnwidth]{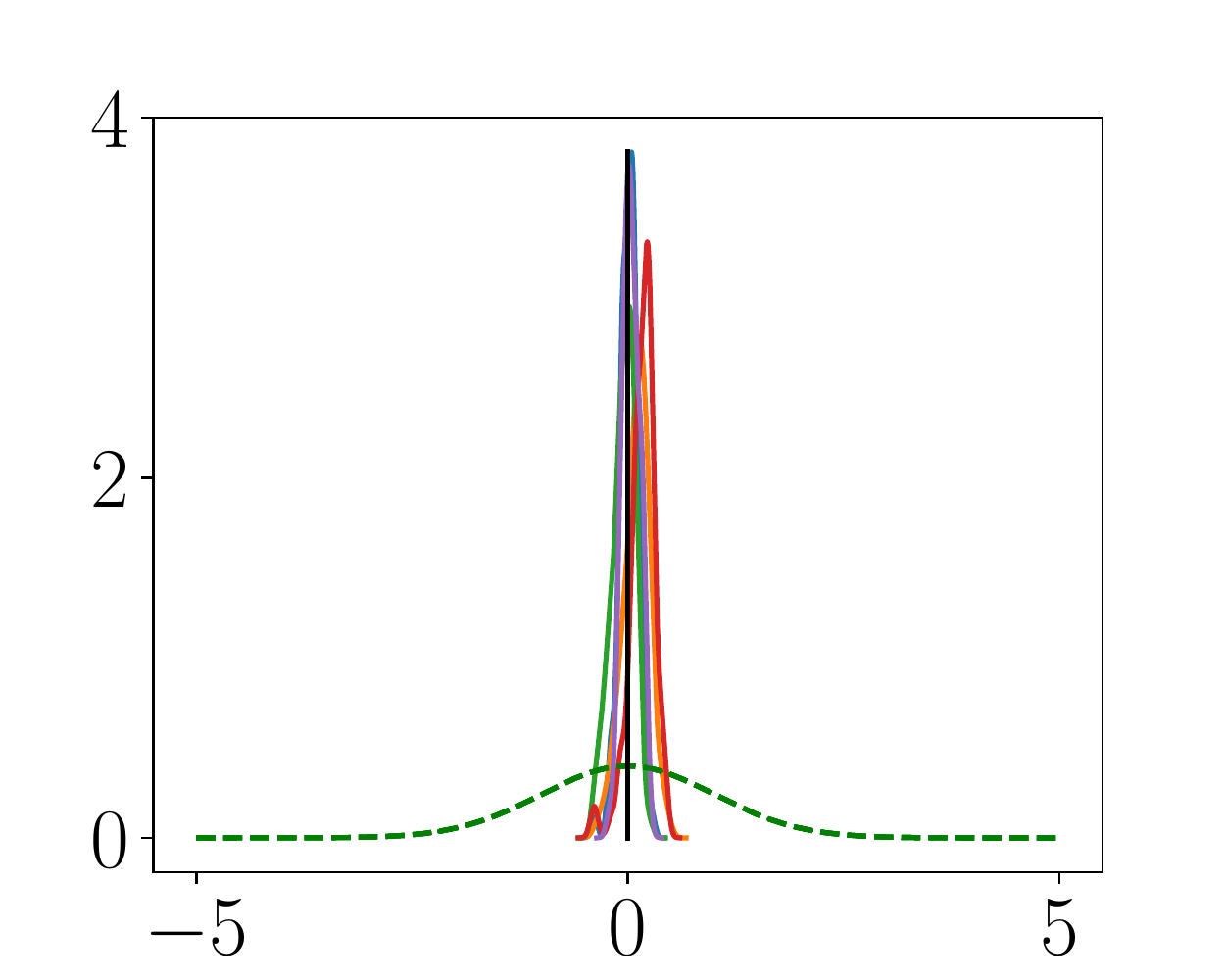}}\qquad
\subfigure[$\tilde{\gamma}$ (MLP pre)]{\includegraphics[width=.18\columnwidth]{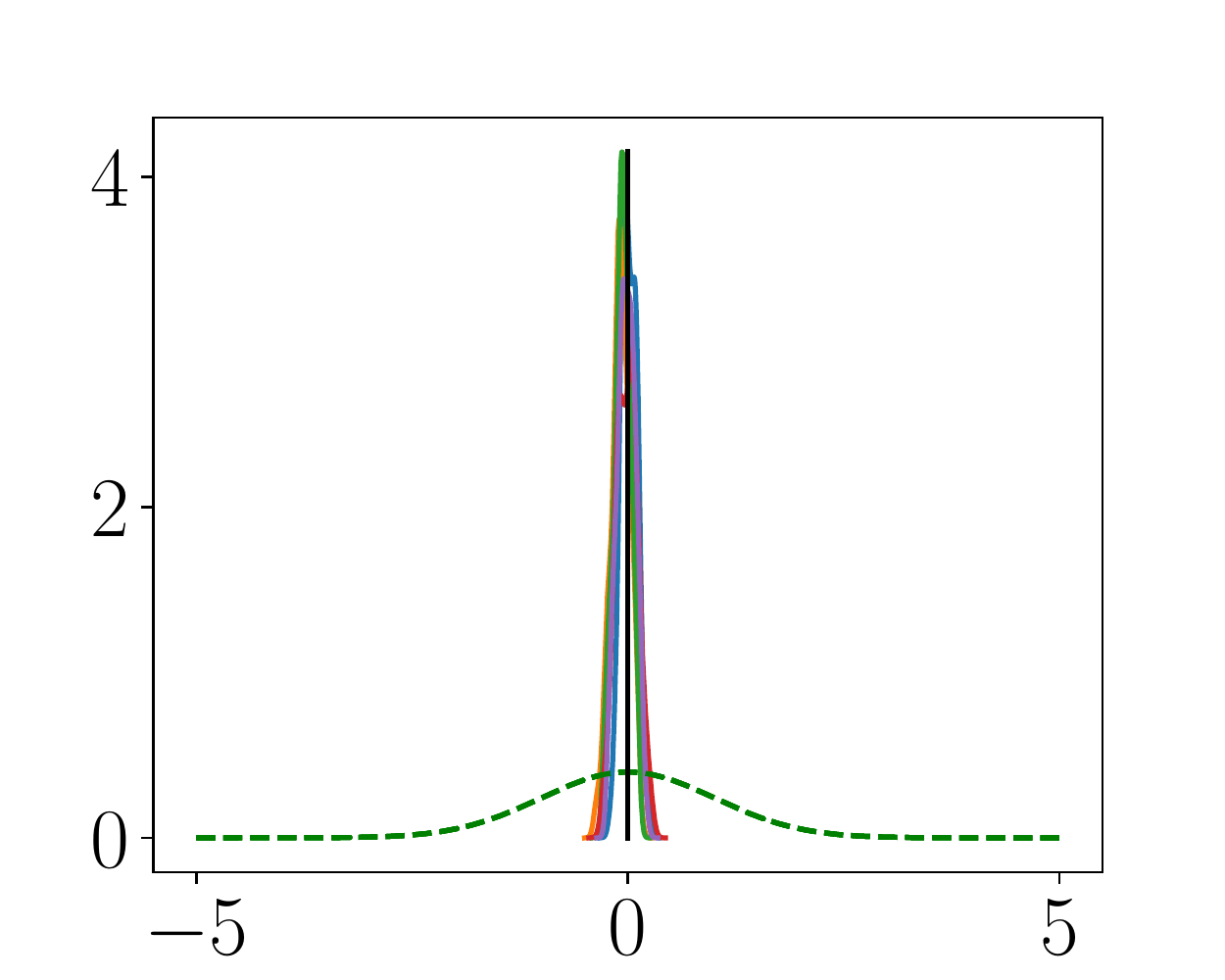}}\qquad
\subfigure[$\tilde{\gamma}$ (PEN-0)]{\includegraphics[width=.18\columnwidth]{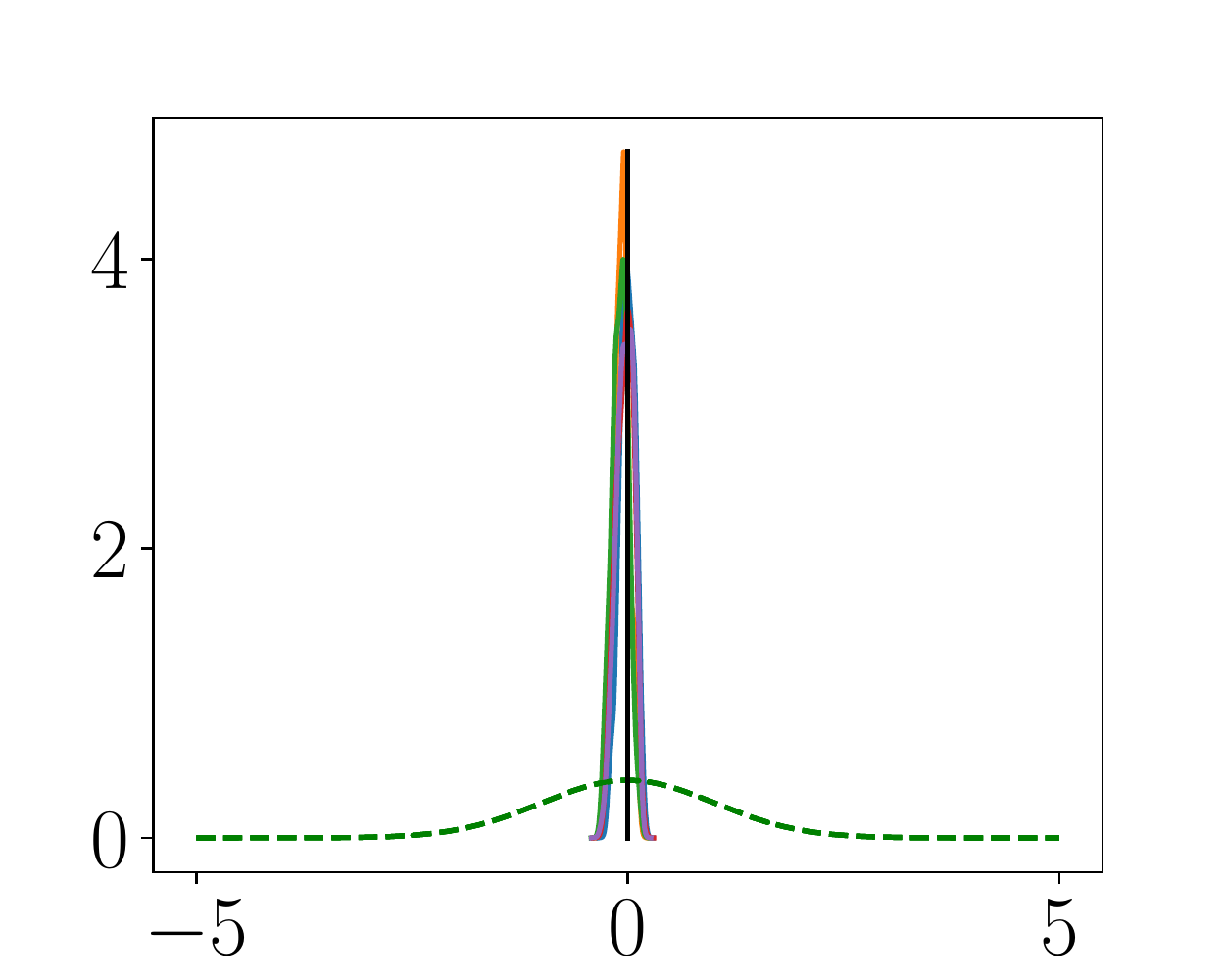}}
\\
\subfigure[$\tilde{\delta}$ (Handpicked) ]{\includegraphics[width=.18\columnwidth]{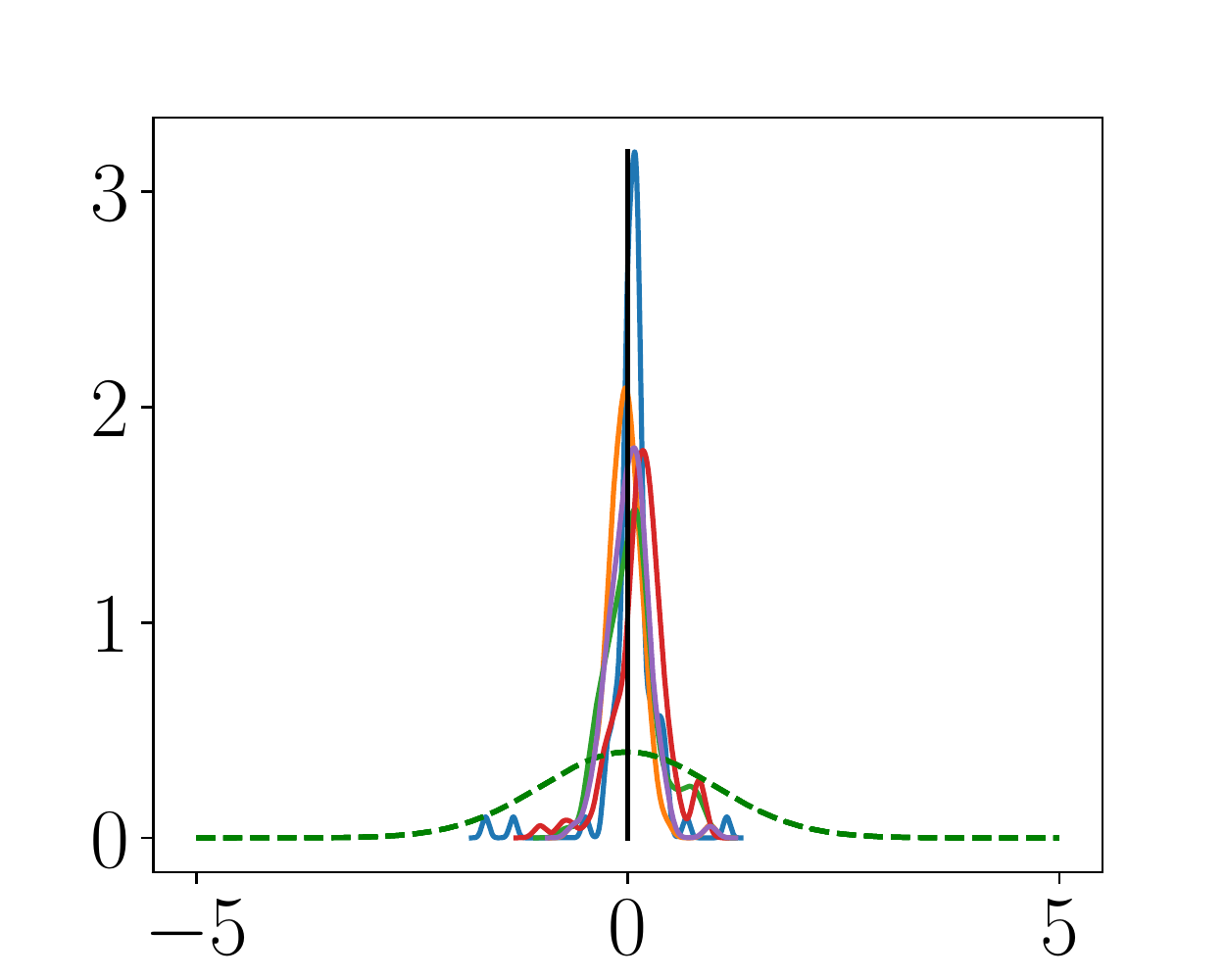}}\qquad
\subfigure[$\tilde{\delta}$ (MLP large)]{\includegraphics[width=.18\columnwidth]{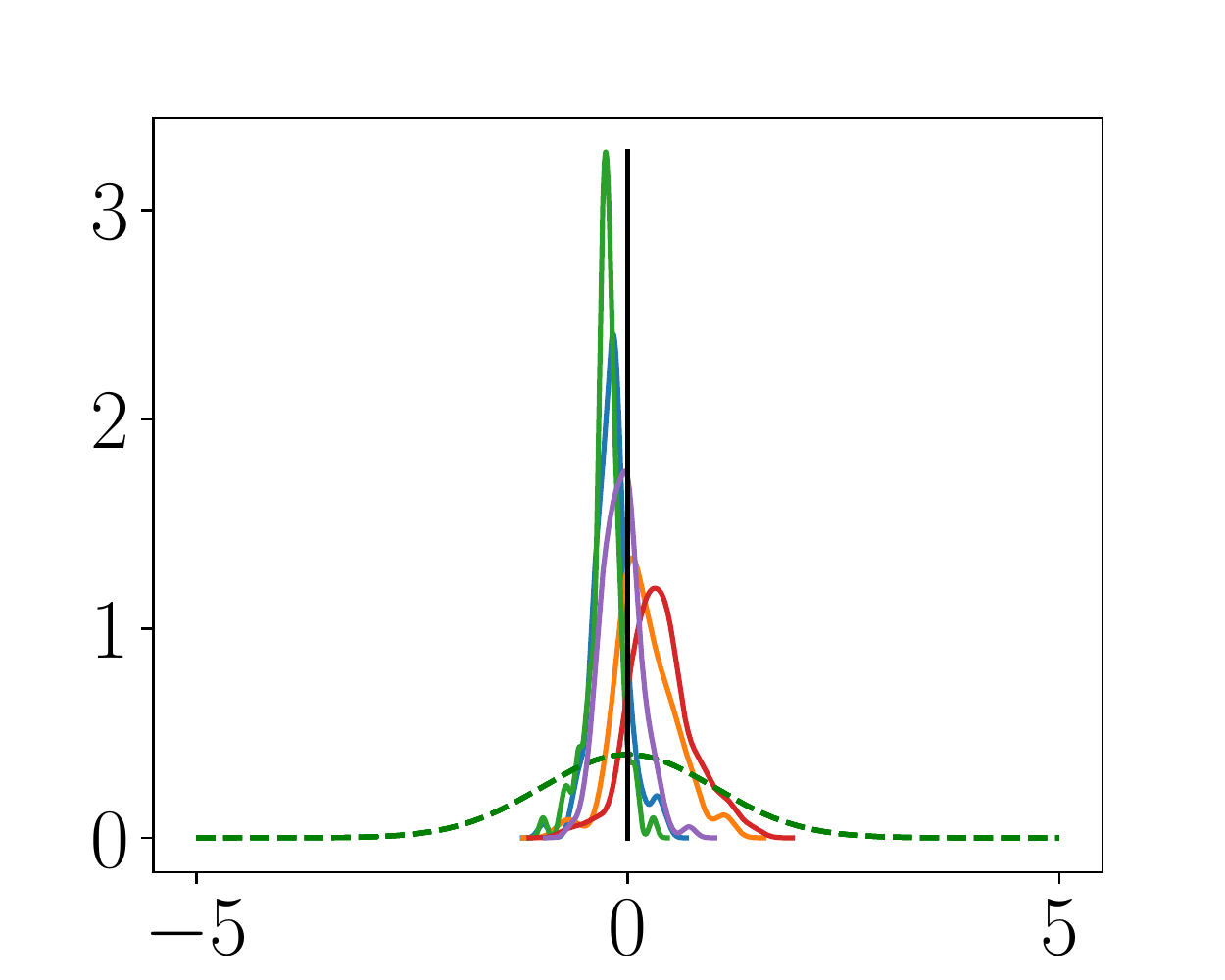}}\qquad
\subfigure[$\tilde{\delta}$ (MLP pre)]{\includegraphics[width=.18\columnwidth]{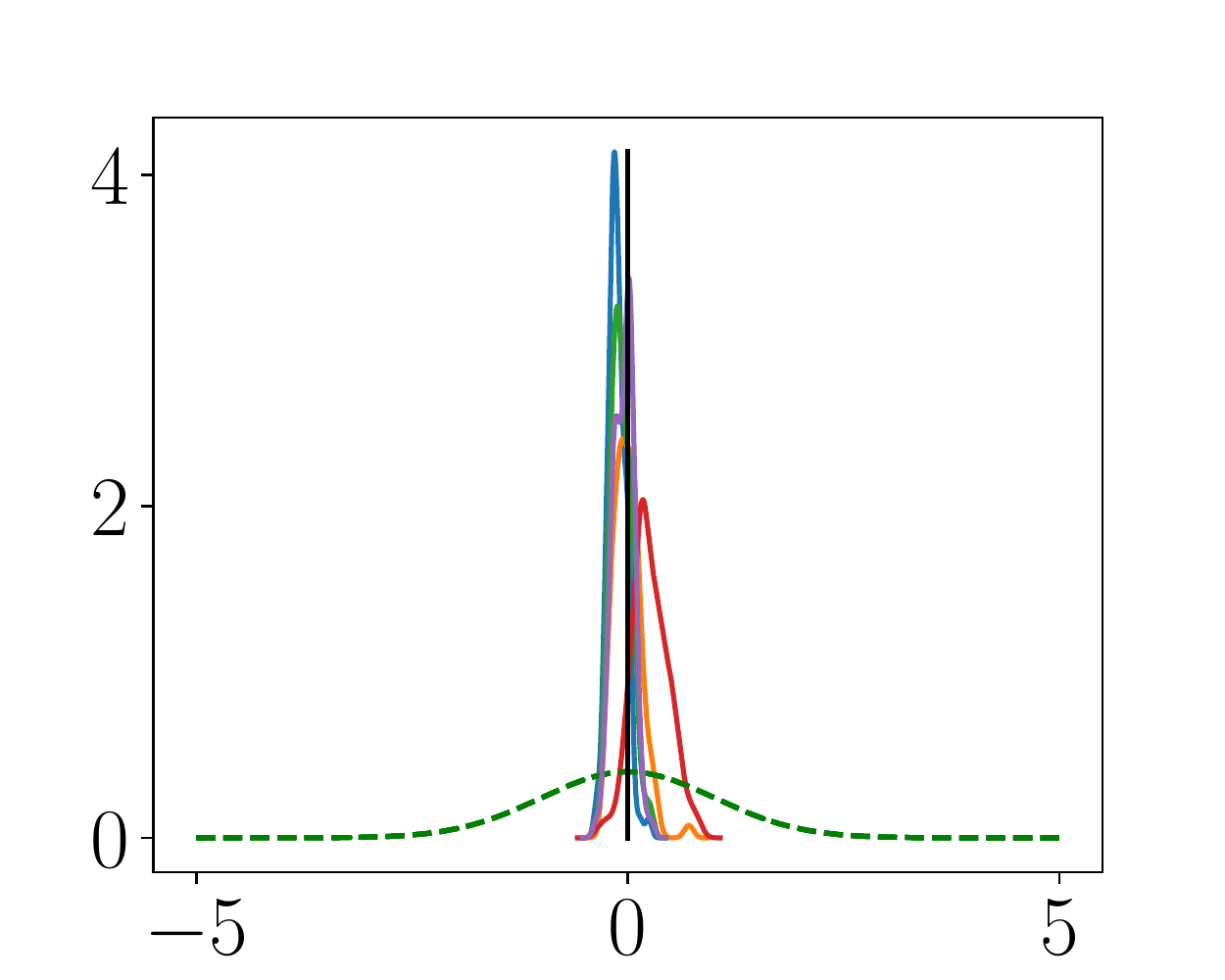}}\qquad
\subfigure[$\tilde{\delta}$ (PEN-0)]{\includegraphics[width=.18\columnwidth]{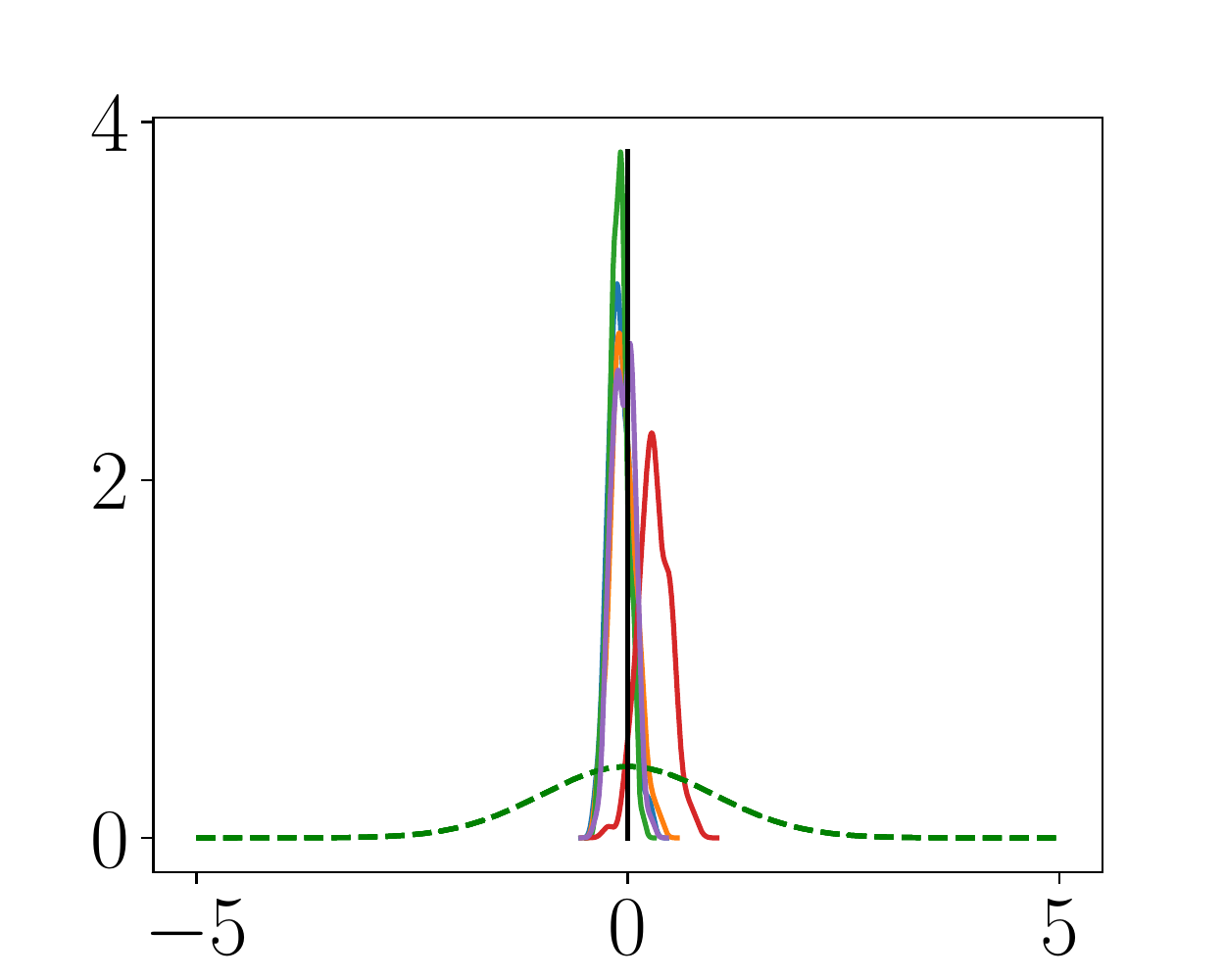}}

\caption{Results for $\alpha$-stable distribution: Approximate marginal ABC posteriors. Results obtained using $5 \cdot 10^5$ training data observations. The green dashed line is the prior distribution. The colored lines show posteriors from 5 independent experiments. These posteriors are not cherry-picked.}
 \label{fig:alphastableposteriors}
\vskip -0.2in
\end{figure}

\subsection{Autoregressive time series model}\label{sec:AR2}

An autoregressive time series model of order two (AR(2)) follows:

$$y_l = \theta_1 y_{l-1} + \theta_2 y_{l-2} + \xi_l,\qquad \xi_l \sim N(0,1).$$

The AR(2) model is identifiable if the following are fulfilled: $\theta_2 < 1 + \theta_1, \theta_2 < 1 - \theta_1, \theta_2 > -1$ \cite{fuller1976introduction}. We let the resulting triangle define the uniform prior for the model. The ground-truth parameters for this simulation study are set to $\theta = [0.2, -0.13]$, and the data size is $M=100$.
AR(2) is a Markov model, hence and the requirement for PEN-$d$ with $d > 0$ is fulfilled.

We compare five methods for computing the summaries: (i) handpicked summary statistics, i.e. $S(y) = [\gamma(y,1), \gamma(y,2), \gamma(y,3), \gamma(y,4), \gamma(y,5)]$ ($\gamma(y,i)$ is autocovariance at lag $i$), which are reasonable summary statistics since autocovariances are normally employed in parameter estimation for autoregressive models, for instance when using the Yule–Walker equations; (ii) ``MLP small'' network; (iii) ``MLP large''; (iv) PEN-0 (DeepSets); and (v) PEN-2. Since AR(2) is a time series model it makes sense to use PEN-2, and PEN-0 results are reported only in the interest of comparison. Here we do not consider the ``MLP pre'' method used in Section \ref{sec:gandk} and \ref{sec:alphstable}, since the empirical distribution function does not have any reasonable meaning for time series data.
The likelihood function for AR(2) is known and we can therefore sample from the true posterior using MCMC.

Results are in \cref{fig:res_ar2}. PEN-2 outperforms MLP, for example we can see that the precision achieved when PEN-2 is trained on $10^3$ training observations can be achieved by MLP when trained on $10^5$ observations, implying an improvement of a $10^2$ factor. Approximate and exact posteriors are in \cref{fig:ar2_approx_posteriors} and we conclude that posteriors for both MLP and PEN-2 are similar to the true posterior when many training observations are used. However, the approximate posterior for MLP degrades significantly when the number of  training observations is reduced and is very uninformative with $10^3$ and even with $10^4$ observations, while for PEN-2 the quality of the approximate posterior distribution is only marginally reduced.
\begin{figure}[ht]
\vskip 0.2in
\begin{center}
\centerline{\includegraphics[width=1\columnwidth]{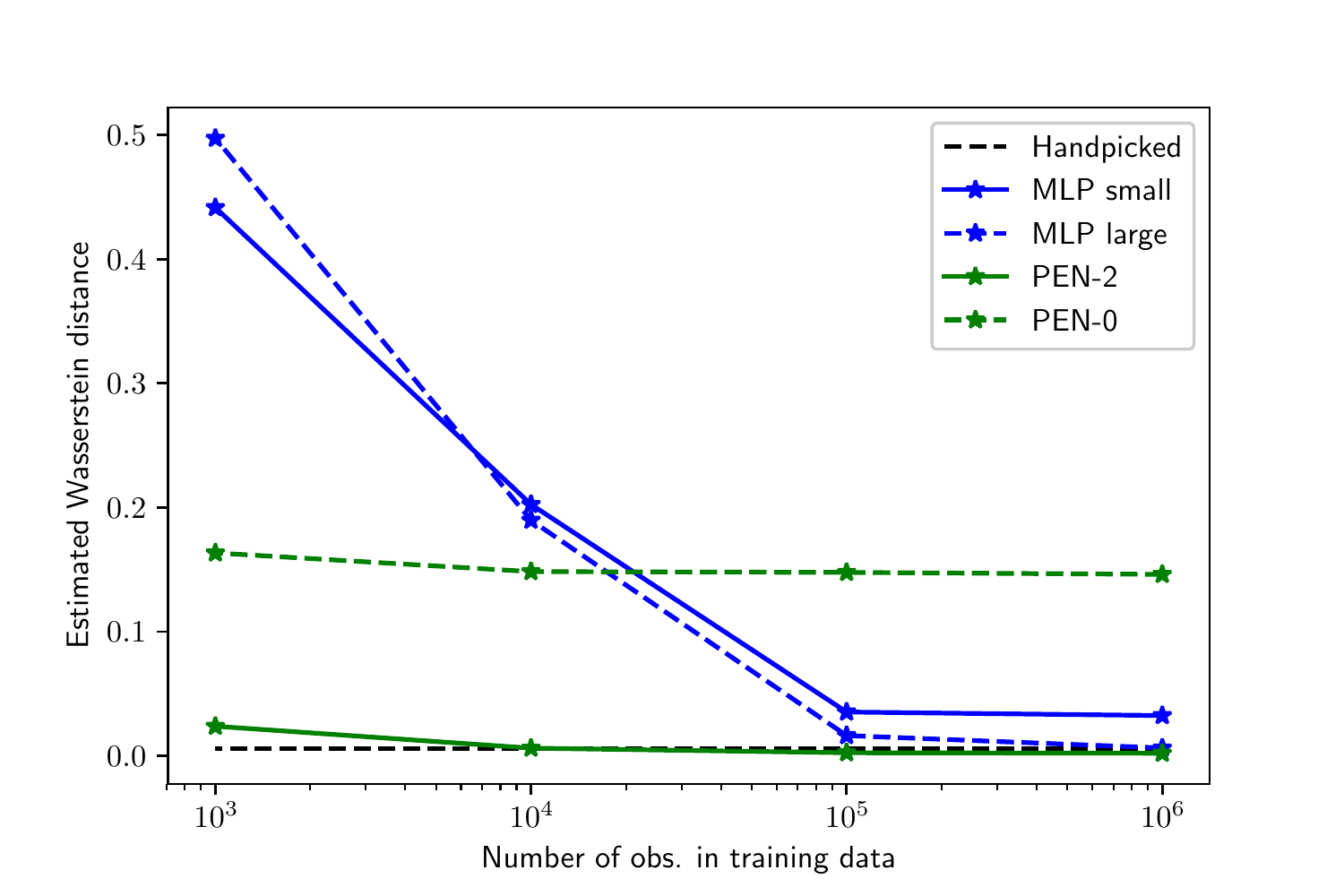}}
\caption{Results for AR(2) model: Estimated Wasserstein distances (mean over 100 data sets) when comparing the true posterior with ABC posteriors, for varying sizes of training data when using DNN models.} 
\label{fig:res_ar2}
\end{center}
\vskip -0.2in
\end{figure}
\begin{figure}[ht]
\vskip 0.2in
\centering
\subfigure[Handpicked \textcolor{white}{some extra text ;)}]{\includegraphics[width=.27\columnwidth]{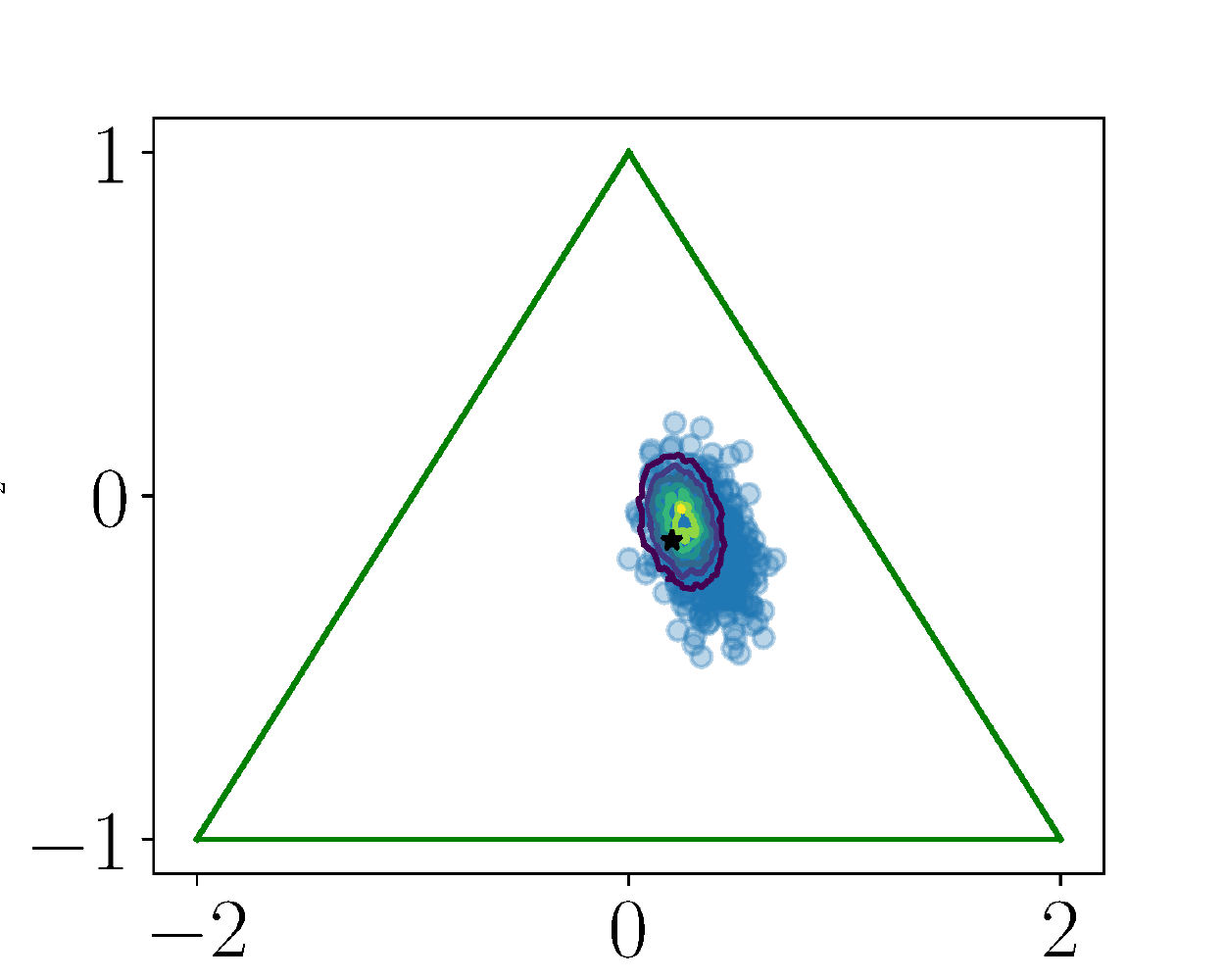}}\qquad
\subfigure[MLP large ($10^6$)]{\includegraphics[width=.27\columnwidth]{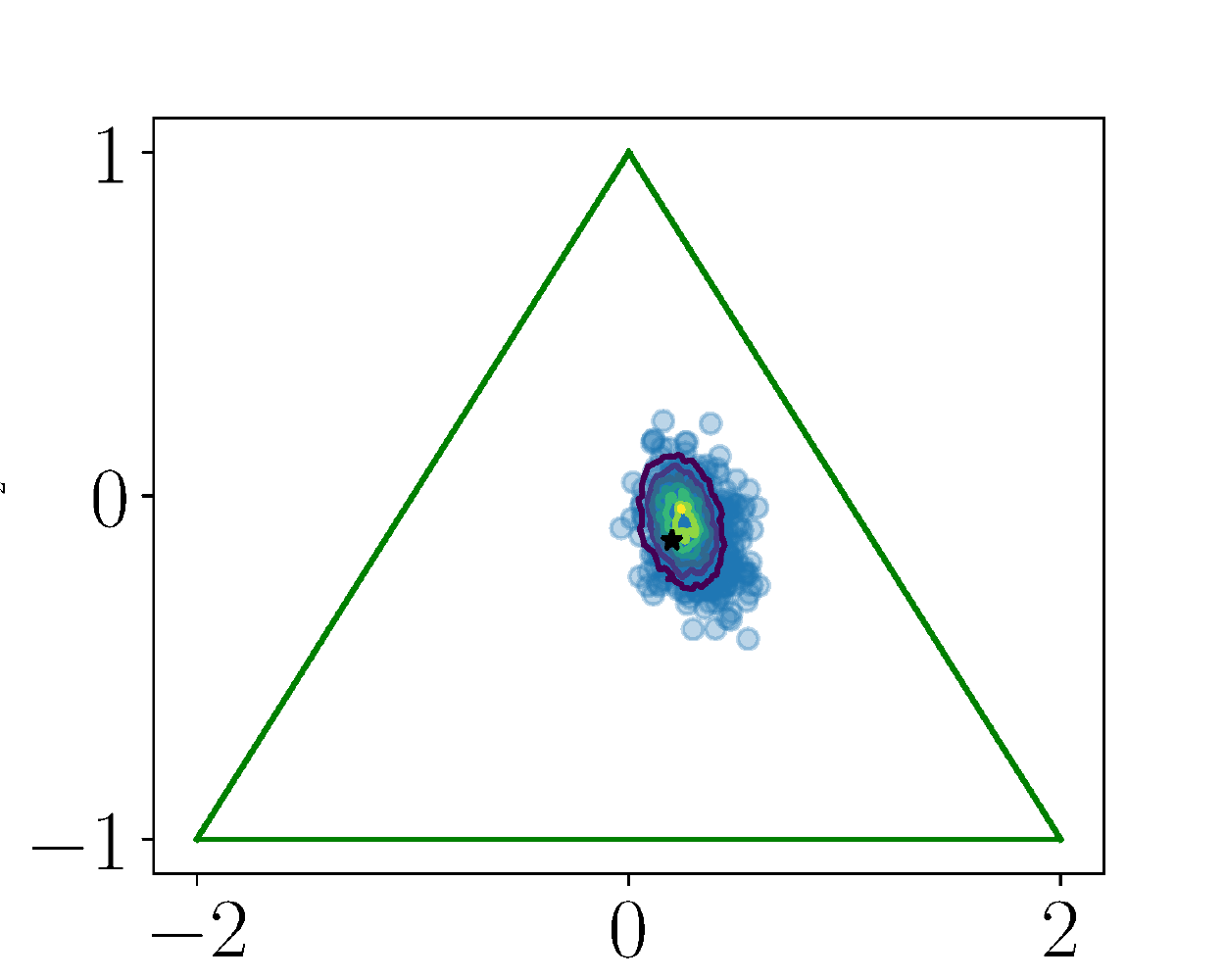}}\qquad
\subfigure[PEN-2 ($10^6$) \textcolor{white}{some extra text ;)}]{\includegraphics[width=.27\columnwidth]{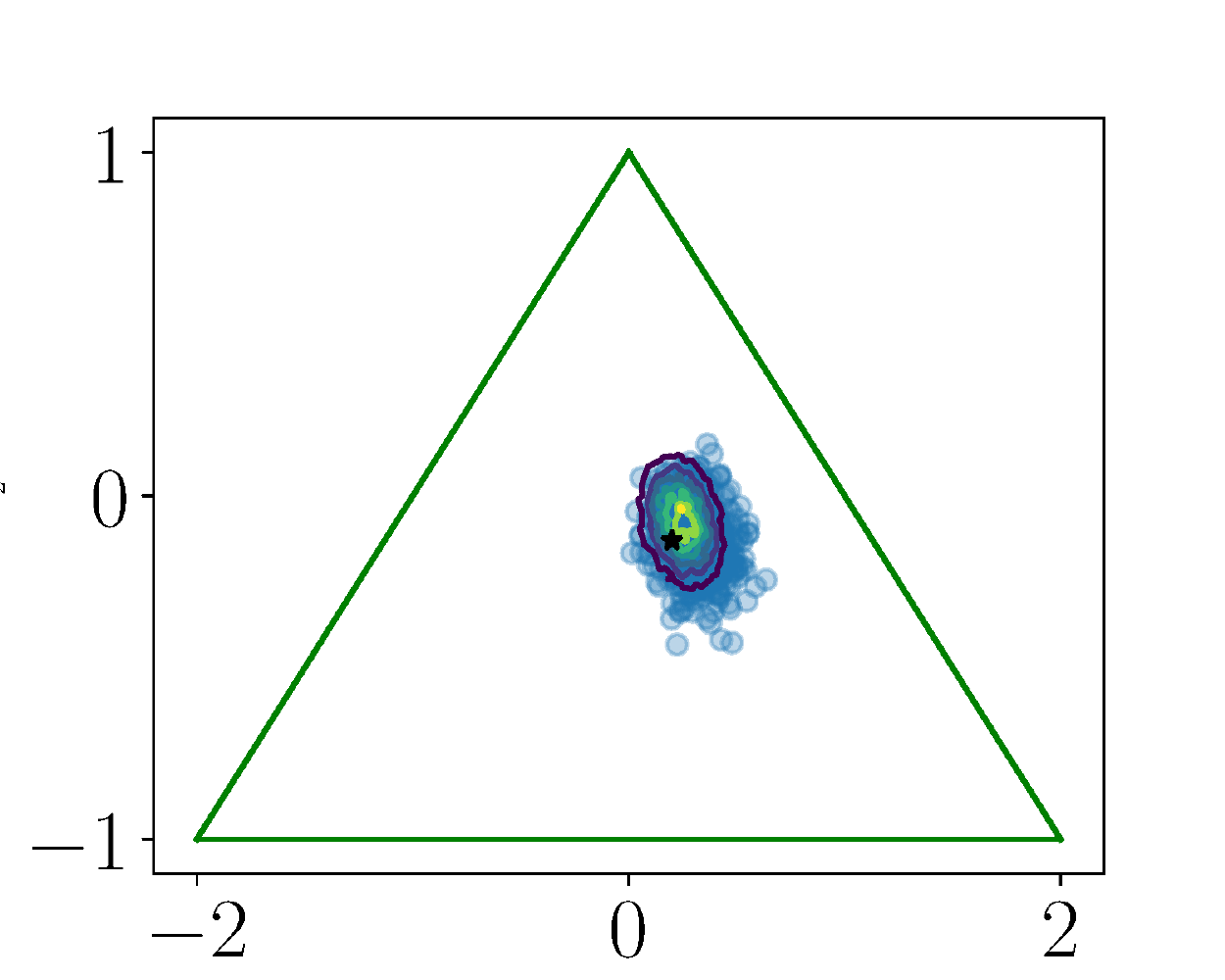}}\qquad
\\
\hspace{2.9cm} 
\subfigure[MLP large ($10^5$)]{\includegraphics[width=.27\columnwidth]{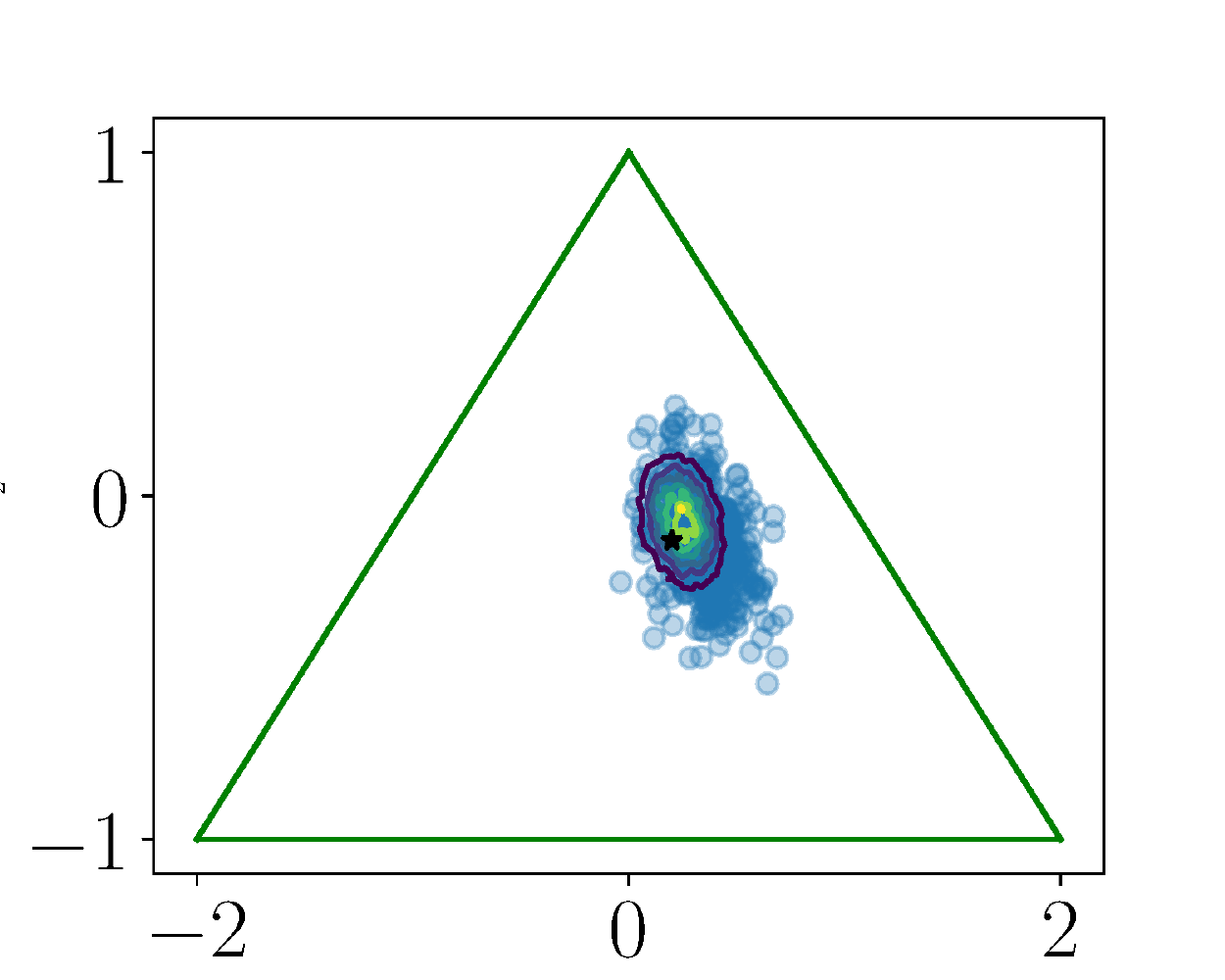}}\qquad
\subfigure[PEN-2 ($10^5$) \textcolor{white}{some extra text ;)}]{\includegraphics[width=.27\columnwidth]{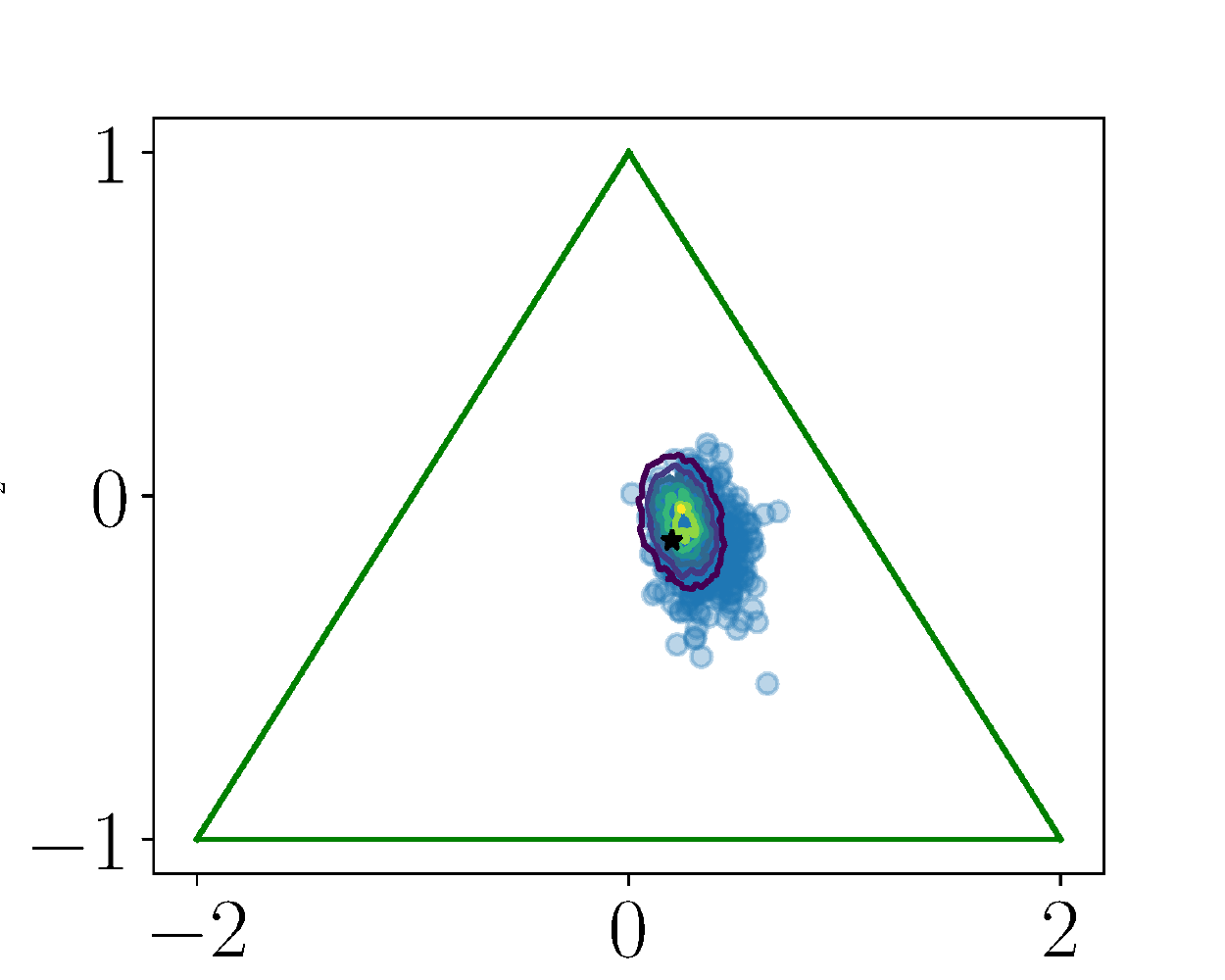}}\qquad
\\
\hspace{2.9cm}
\subfigure[MLP large ($10^4$)]{\includegraphics[width=.27\columnwidth]{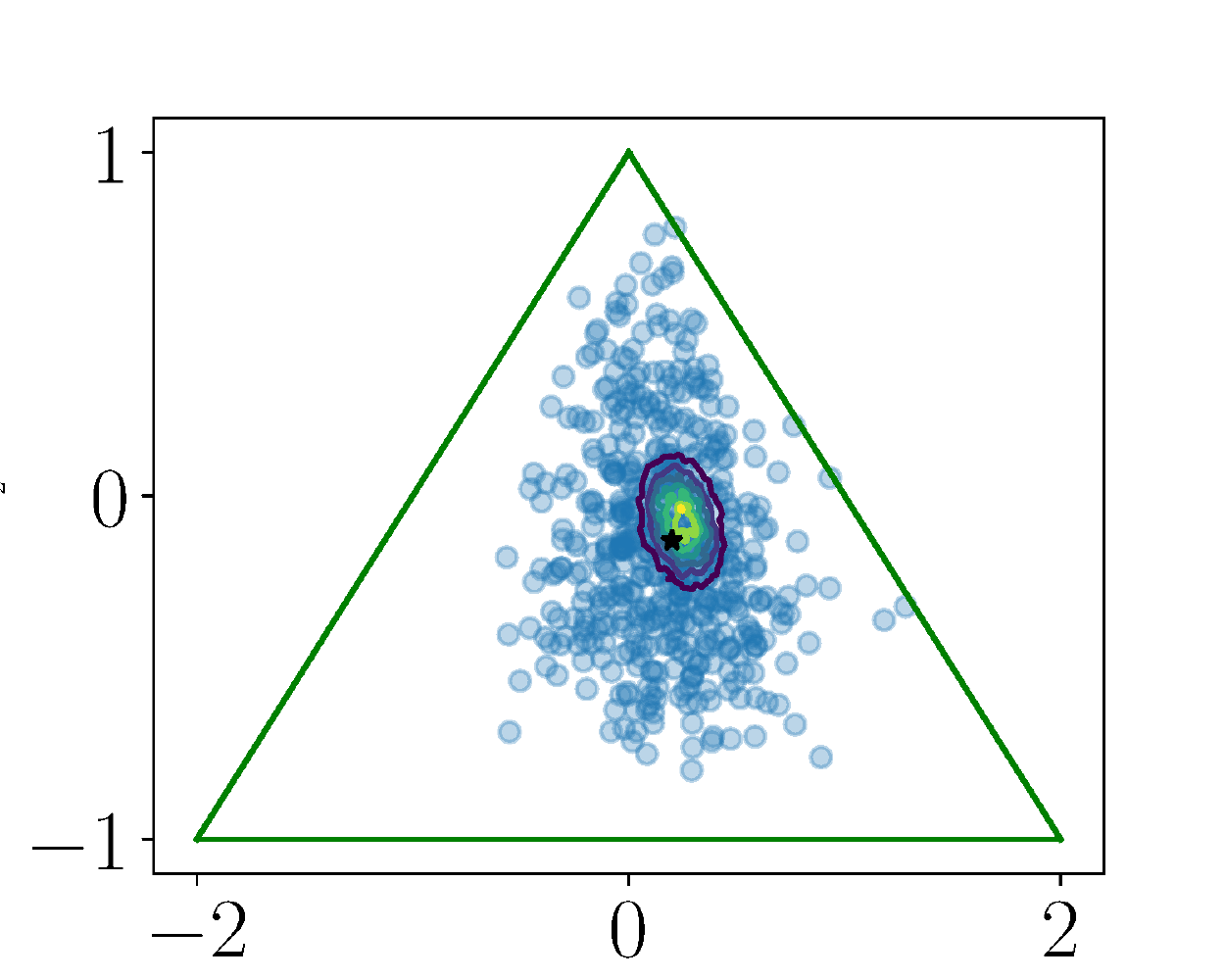}}\qquad
\subfigure[PEN-2 ($10^4$) \textcolor{white}{some extra text ;)}]{\includegraphics[width=.27\columnwidth]{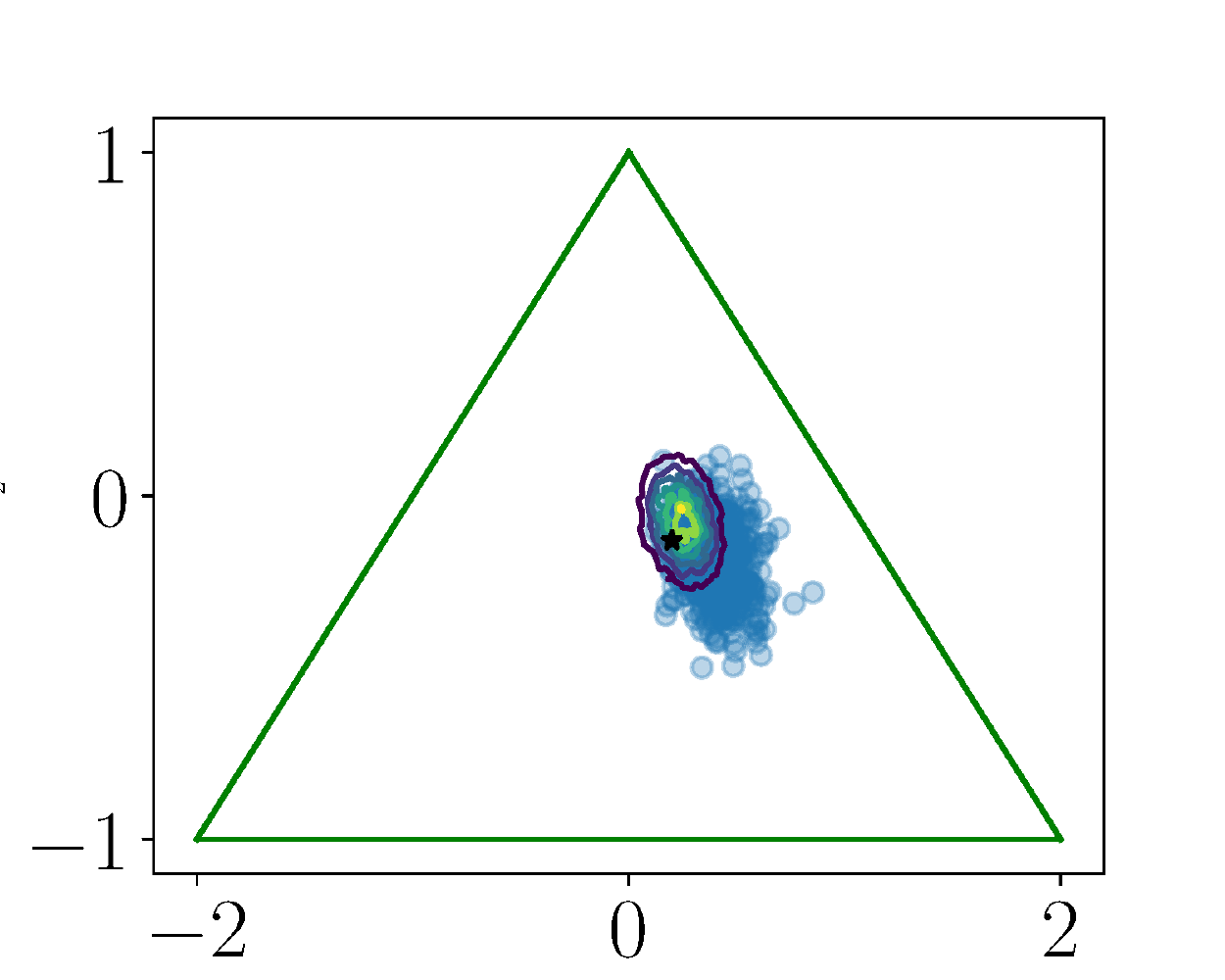}}\qquad
\\
\hspace{2.9cm}
\subfigure[MLP large ($10^3$)]{\includegraphics[width=.27\columnwidth]{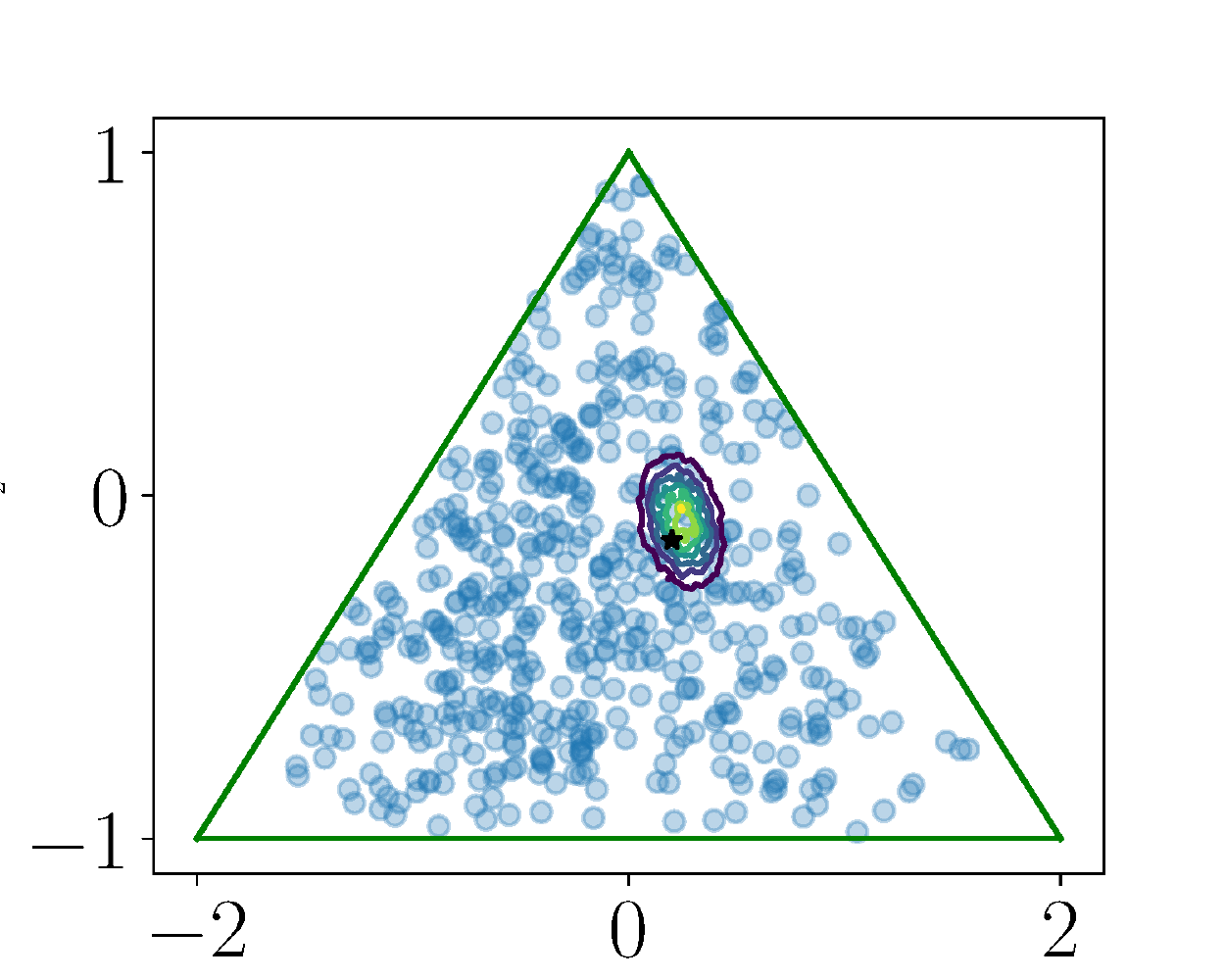}}\qquad
\subfigure[PEN-2 ($10^3$) \textcolor{white}{some extra text ;)}]{\includegraphics[width=.27\columnwidth]{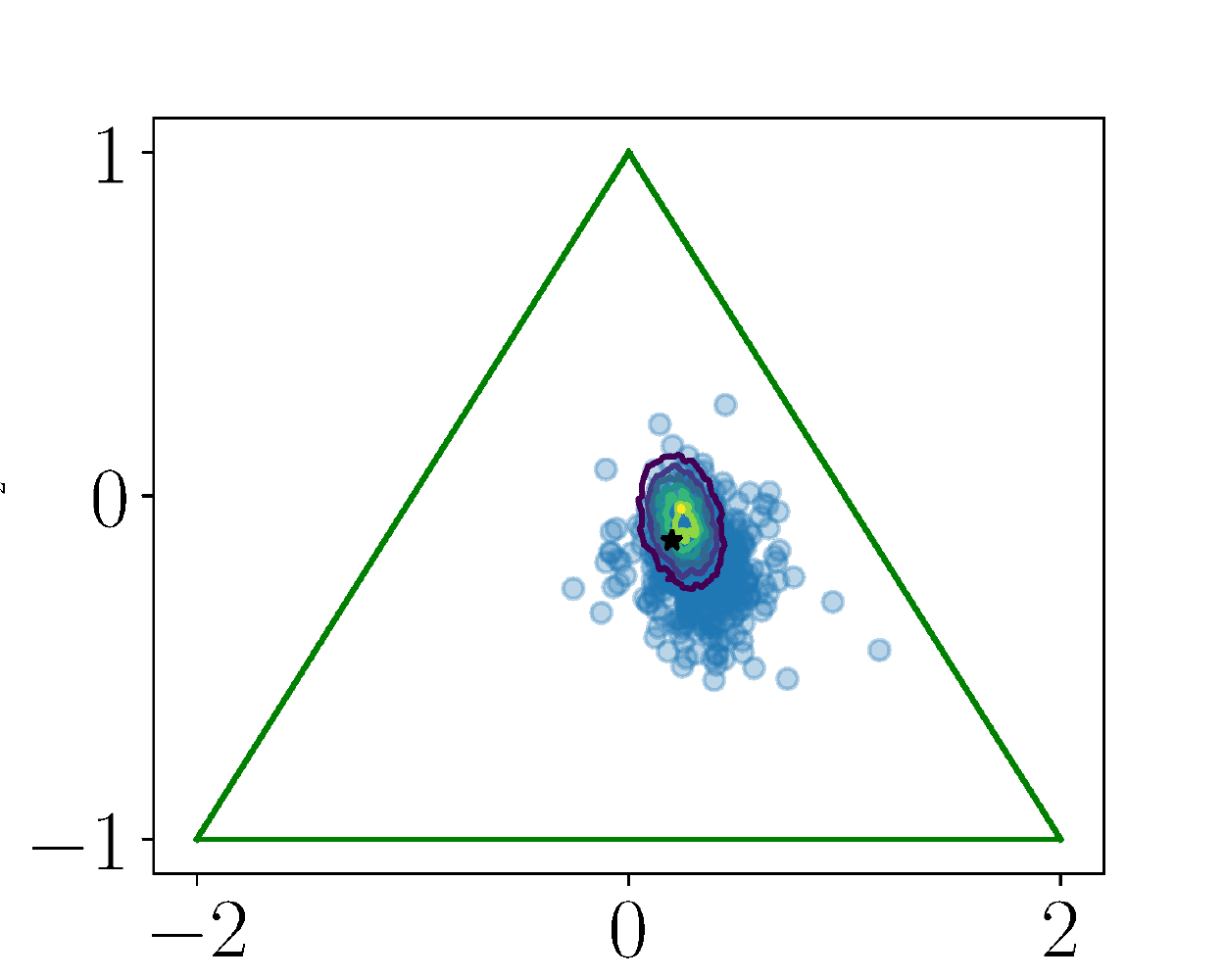}}\qquad
\caption{Results for AR(2) model. The green line indicates the prior distribution, the contour plot is from the exact posterior and the blue dots are 100 samples from the several ABC posteriors. The number in parenthesis indicates number of observations in the training data set. These posteriors are not cherry-picked.}
\label{fig:ar2_approx_posteriors}
\vskip -0.2in
\end{figure}

\subsection{Moving average time series with observational noise model}
We consider a partially observed time series, with latent dynamics given by a moving average MA(2) model and observations perturbed with Gaussian noise:
\begin{align*}
\begin{cases}
y_l = x_l + \xi^y_l, \qquad \xi^y_l \sim N(0, \sigma_\epsilon = 0.3), \\
x_l = \xi_l + \theta_1 \xi^x_{l-1} + \theta_2 \xi^x_{l-2}, \qquad \xi^x_l \sim N(0, 1),
\end{cases}
\end{align*}
where the $\xi^x_l$ and $\xi^y_l$ are all independent.
An MA(2) process without observational noise is identifiable if $\theta_1\in[-2,2]$, $\theta_2\in[-1,1]$, and $\theta_2\pm\theta_1\ge -1$. Same as in \citet{jiang2017learning}, we define a uniform prior over this triangle. We use the same setting as in  \citet{jiang2017learning} and set the ground-truth parameters for the simulation study to $\theta = [0.6, 0.2]$. We only observe $\{y_l\}$ and the number of observations is $M=100$.

The latent dynamics are not Markovian, hence the Markov property required for PEN of order larger than 0 is not fulfilled, however, the quasi-Markov structure of the data might still allow us to successfully use PEN-$d$ with an order $d$ larger than 0. An additional complication is given by the observational noise $\xi_l^y$, further perturbing the dynamics.
Once more, we compare five methods for computing the summary statistics: (i) handpicked summaries $S(y) = [\gamma(y,1), \gamma(y,2)]$, i.e. we follow \citet{jiang2017learning}; (ii) ``MLP small''; (iii) ``MLP large''; (iv) PEN-0 (DeepSets); and (v) PEN-10. Same as for the AR(2) example, here PEN-0 results are reported only in the interest of a comparison with PEN-10, as for a time-series model it is expected from PEN-0 to be suboptimal.
Also in this case the likelihood function is available, and we can compute the true posterior distribution. Once more, we compare the approximate posteriors to the true posterior over 100 different data sets, see \cref{fig:res_ma2}. We conclude that PEN-10 performs slightly better than MLP when the training data set is large, and that PEN-10 outperforms MLP when we restrict the size of the training data. Once more, we notice that PEN-10 implies a factor $\geq 10$ in terms of savings on the size of the training data.

\begin{figure}[ht]
\vskip 0.2in
\begin{center}
\centerline{\includegraphics[width=1\columnwidth]{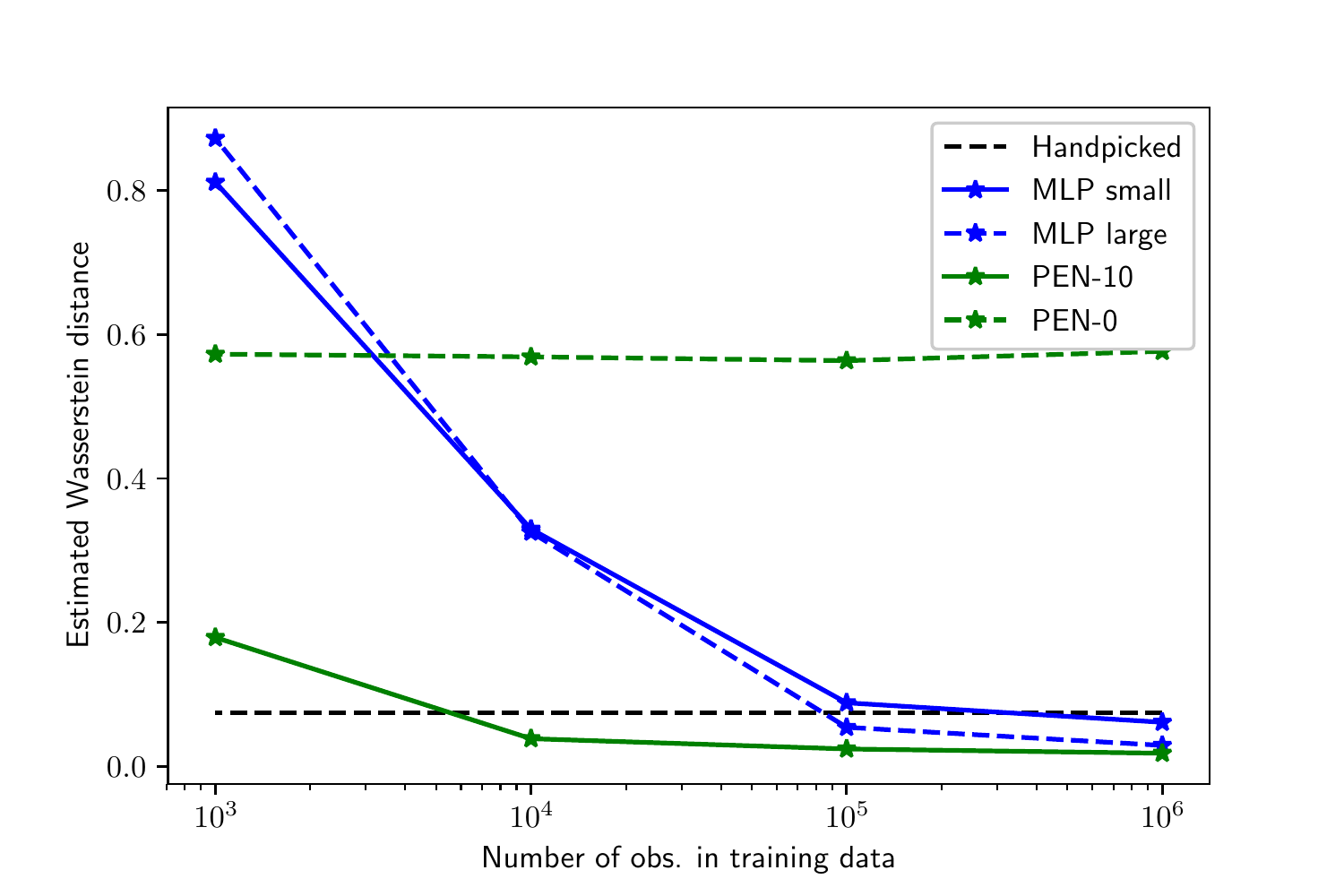}}
\caption{Results for MA(2) model: Estimated Wasserstein distances (mean over 100 data sets) when comparing the true posterior with ABC posteriors.}
\label{fig:res_ma2}
\end{center}
\vskip -0.2in
\end{figure}

\section{Discussion} \label{sec:discussion}

Simulation experiments show that our partially exchangeable networks (PENs) achieve competitive results in learning summary statistics for use in ABC algorithms, outperforming the other deep learning methods that we have considered. Moreover, PENs require much smaller training data to achieve the same inference accuracy of competitors: in our experiments a reduction factor of order $10$ to $10^2$ was observed.

As mentioned in Section \ref{sec:abc}, in this work we were not focused on the specific ABC algorithm used for sampling, but only on learning summary statistics for ABC. However, in future work we plan to use  our  approach  for  constructing  summary  statistics  alongside  more  sophisticated variants of ABC methods, such as those which combine ABC with Markov chain  Monte  Carlo \cite{sisson2011likelihood}  or  sequential  techniques \cite{beaumont2009adaptive}.

\citet{murphy2018janossy} recently shed light on some limitations of the DeepSets architecture, and proposed to improve it by replacing the sum fed to the outer network by another pooling techinque called \emph{Janossy pooling}. Since the drawbacks they inspect are also likely to affect our architectures, extending Janossy pooling to the PEN framework might constitute a valuable improvement.


Our experiments show that the performance of the MLP networks using different choices for the number of weights is quite similar, and that PEN outperforms MLP even when MLP has access to a larger number of weights compared to PEN.
The main insight is that PENs by design incorporate the (partial) exchangeability property of the data, whereas the MLPs have to learn this property.
Exchangeability and partial exchangeability can in principle be expressed
in an MLP, but for small data sets these properties will be difficult to learn, and we expect that the model will overfit to the training data. One approach to alleviate this problem for MLPs is to perform data augmentation. However, it is not straightforward to perform data augmentation for continuous Markovian data, unless we have access to the underlying data generating process. In ABC the assumption is that we do have access to this process, but data generation may be computational expensive, and in a more general application we may not have access to the process.



Although we have applied the PEN architecture to the problem of learning summary statistics for ABC, notice that PEN is a general architecture and could be used for other applications. One example would be time series classification.

The main limitation for PEN is that it is designed for Markovian data or, when considering the special case of DeepSets (i.e. PEN-0), for exchangeable data. However, in the MA(2) example we achieve good inference results even though the MA(2) model is itself non-Markovian \textit{and} observations are perturbed with measurement noise.


\section*{Acknowledgements}

Research was partially supported by the Swedish Research Council (VR grant 2013-05167). We would also like to thank Joachim Hein and colleagues at LUNARC, Lund University, for helping out on setting up the GPU environment used for the simulations.

\bibliography{references}

\begin{thebibliography}{33}
\providecommand{\natexlab}[1]{#1}
\providecommand{\url}[1]{\texttt{#1}}
\expandafter\ifx\csname urlstyle\endcsname\relax
  \providecommand{\doi}[1]{doi: #1}\else
  \providecommand{\doi}{doi: \begingroup \urlstyle{rm}\Url}\fi

\bibitem[Allingham et~al.(2009)Allingham, King, and
  Mengersen]{allingham2009bayesian}
Allingham, D., King, R., and Mengersen, K.~L.
\newblock Bayesian estimation of quantile distributions.
\newblock \emph{Statistics and Computing}, 19\penalty0 (2):\penalty0 189--201,
  2009.

\bibitem[Beaumont et~al.(2002)Beaumont, Zhang, and
  Balding]{beaumont2002approximate}
Beaumont, M.~A., Zhang, W., and Balding, D.~J.
\newblock Approximate {B}ayesian computation in population genetics.
\newblock \emph{Genetics}, 162\penalty0 (4):\penalty0 2025--2035, 2002.

\bibitem[Beaumont et~al.(2009)Beaumont, Cornuet, Marin, and
  Robert]{beaumont2009adaptive}
Beaumont, M.~A., Cornuet, J.-M., Marin, J.-M., and Robert, C.~P.
\newblock Adaptive approximate {B}ayesian computation.
\newblock \emph{Biometrika}, 96\penalty0 (4):\penalty0 983--990, 2009.

\bibitem[Bezanson et~al.(2017)Bezanson, Edelman, Karpinski, and
  Shah]{bezanson2017julia}
Bezanson, J., Edelman, A., Karpinski, S., and Shah, V.~B.
\newblock Julia: A fresh approach to numerical computing.
\newblock \emph{SIAM review}, 59\penalty0 (1):\penalty0 65--98, 2017.

\bibitem[Bloem-Reddy \& Teh(2019)Bloem-Reddy and Teh]{bloem2019}
Bloem-Reddy, B. and Teh, Y.~W.
\newblock Probabilistic symmetry and invariant neural networks.
\newblock \emph{arXiv:1901.06082}, 2019.

\bibitem[Blum et~al.(2013)Blum, Nunes, Prangle, Sisson,
  et~al.]{blum2013comparative}
Blum, M.~G., Nunes, M.~A., Prangle, D., Sisson, S.~A., et~al.
\newblock A comparative review of dimension reduction methods in approximate
  {B}ayesian computation.
\newblock \emph{Statistical Science}, 28\penalty0 (2):\penalty0 189--208, 2013.

\bibitem[Chan et~al.(2018)Chan, Perrone, Spence, Jenkins, Mathieson, and
  Song]{chan2018likelihood}
Chan, J., Perrone, V., Spence, J., Jenkins, P., Mathieson, S., and Song, Y.
\newblock A likelihood-free inference framework for population genetic data
  using exchangeable neural networks.
\newblock In Bengio, S., Wallach, H., Larochelle, H., Grauman, K.,
  Cesa-Bianchi, N., and Garnett, R. (eds.), \emph{Advances in Neural
  Information Processing Systems 31}, pp.\  8603--8614. Curran Associates,
  Inc., 2018.

\bibitem[Cornuet et~al.(2008)Cornuet, Santos, Beaumont, Robert, Marin, Balding,
  Guillemaud, and Estoup]{cornuet2008inferring}
Cornuet, J.-M., Santos, F., Beaumont, M.~A., Robert, C.~P., Marin, J.-M.,
  Balding, D.~J., Guillemaud, T., and Estoup, A.
\newblock Inferring population history with {DIY ABC}: a user-friendly approach
  to approximate {B}ayesian computation.
\newblock \emph{Bioinformatics}, 24\penalty0 (23):\penalty0 2713--2719, 2008.

\bibitem[Creel(2017)]{creel2017neural}
Creel, M.
\newblock Neural nets for indirect inference.
\newblock \emph{Econometrics and Statistics}, 2:\penalty0 36--49, 2017.

\bibitem[de~Finetti(1929)]{definetti1929}
de~Finetti, B.
\newblock Funzione caratteristica di un fenomeno aleatorio.
\newblock In \emph{Atti del Congresso Internazionale dei Matematici: Bologna
  dal 3 al 10 di settembre 1928}, pp.\  179--190, 1929.

\bibitem[Diaconis(1988)]{diaconis1988}
Diaconis, P.
\newblock Recent progress on de {Finetti}'s notions of exchangeability.
\newblock \emph{Bayesian statistics}, 3:\penalty0 111--125, 1988.

\bibitem[Diaconis \& Freedman(1980)Diaconis and Freedman]{diaconis1980}
Diaconis, P. and Freedman, D.
\newblock de {F}inetti's theorem for {M}arkov chains.
\newblock \emph{The Annals of Probability}, pp.\  115--130, 1980.

\bibitem[Fearnhead \& Prangle(2012)Fearnhead and
  Prangle]{fearnhead2012constructing}
Fearnhead, P. and Prangle, D.
\newblock Constructing summary statistics for approximate bayesian computation:
  semi-automatic approximate {B}ayesian computation.
\newblock \emph{Journal of the Royal Statistical Society: Series B},
  74\penalty0 (3):\penalty0 419--474, 2012.

\bibitem[Flamary \& Courty(2017)Flamary and Courty]{flamary2017pot}
Flamary, R. and Courty, N.
\newblock {POT} {P}ython optimal transport library, 2017.
\newblock URL \url{https://github.com/rflamary/POT}.

\bibitem[Fuller(1976)]{fuller1976introduction}
Fuller, W.~A.
\newblock \emph{Introduction to time series analysis}.
\newblock New York: John Wiley \& Sons, 1976.

\bibitem[Jasra(2015)]{jasra2015approximate}
Jasra, A.
\newblock Approximate {B}ayesian computation for a class of time series models.
\newblock \emph{International Statistical Review}, 83\penalty0 (3):\penalty0
  405--435, 2015.

\bibitem[Jiang et~al.(2017)Jiang, Wu, Zheng, and Wong]{jiang2017learning}
Jiang, B., Wu, T.-y., Zheng, C., and Wong, W.~H.
\newblock Learning summary statistic for approximate {B}ayesian computation via
  deep neural network.
\newblock \emph{Statistica Sinica}, pp.\  1595--1618, 2017.

\bibitem[Marin et~al.(2012)Marin, Pudlo, Robert, and
  Ryder]{marin2012approximate}
Marin, J.-M., Pudlo, P., Robert, C.~P., and Ryder, R.~J.
\newblock Approximate {B}ayesian computational methods.
\newblock \emph{Statistics and Computing}, 22\penalty0 (6):\penalty0
  1167--1180, 2012.

\bibitem[Minsky \& Papert(1988)Minsky and Papert]{minsky1988perceptrons}
Minsky, M. and Papert, S.
\newblock \emph{Perceptrons (expanded edition) MIT Press}.
\newblock 1988.

\bibitem[Murphy et~al.(2019)Murphy, Srinivasan, Rao, and
  Ribeiro]{murphy2018janossy}
Murphy, R.~L., Srinivasan, B., Rao, V., and Ribeiro, B.
\newblock Janossy pooling: Learning deep permutation-invariant functions for
  variable-size inputs.
\newblock In \emph{International Conference on Learning Representations}, 2019.

\bibitem[Ong et~al.(2018)Ong, Nott, Tran, Sisson, and
  Drovandi]{ong2018variational}
Ong, V.~M., Nott, D.~J., Tran, M.-N., Sisson, S.~A., and Drovandi, C.~C.
\newblock Variational bayes with synthetic likelihood.
\newblock \emph{Statistics and Computing}, 28\penalty0 (4):\penalty0 971--988,
  2018.

\bibitem[Peters et~al.(2012)Peters, Sisson, and Fan]{peters2012likelihood}
Peters, G.~W., Sisson, S.~A., and Fan, Y.
\newblock Likelihood-free bayesian inference for $\alpha$-stable models.
\newblock \emph{Computational Statistics \& Data Analysis}, 56\penalty0
  (11):\penalty0 3743--3756, 2012.

\bibitem[Picchini(2014)]{picchini2014inference}
Picchini, U.
\newblock Inference for {SDE} models via approximate {B}ayesian computation.
\newblock \emph{Journal of Computational and Graphical Statistics}, 23\penalty0
  (4):\penalty0 1080--1100, 2014.

\bibitem[Picchini \& Anderson(2017)Picchini and
  Anderson]{picchini2017approximate}
Picchini, U. and Anderson, R.
\newblock Approximate maximum likelihood estimation using data-cloning {ABC}.
\newblock \emph{Computational Statistics \& Data Analysis}, 105:\penalty0
  166--183, 2017.

\bibitem[Prangle(2015)]{prangle2015summary}
Prangle, D.
\newblock Summary statistics in approximate {B}ayesian computation.
\newblock \emph{arXiv:1512.05633}, 2015.

\bibitem[Prangle(2017)]{prangle2017gk}
Prangle, D.
\newblock gk: An {R} package for the g-and-k and generalised g-and-h
  distributions.
\newblock \emph{arXiv:1706.06889}, 2017.

\bibitem[Pritchard et~al.(1999)Pritchard, Seielstad, Perez-Lezaun, and
  Feldman]{pritchard1999population}
Pritchard, J.~K., Seielstad, M.~T., Perez-Lezaun, A., and Feldman, M.~W.
\newblock Population growth of human {Y} chromosomes: a study of {Y} chromosome
  microsatellites.
\newblock \emph{Molecular biology and evolution}, 16\penalty0 (12):\penalty0
  1791--1798, 1999.

\bibitem[Ravanbakhsh et~al.(2017)Ravanbakhsh, Schneider, and
  Poczos]{ravanbakhsh2017deep}
Ravanbakhsh, S., Schneider, J., and Poczos, B.
\newblock Deep learning with sets and point clouds.
\newblock \emph{International Conference on Learning Representations (ICLR) -
  workshop track}, 2017.

\bibitem[Shawe-Taylor(1989)]{shawe1989}
Shawe-Taylor, J.
\newblock Building symmetries into feedforward networks.
\newblock In \emph{Artificial Neural Networks, 1989., First IEE International
  Conference on (Conf. Publ. No. 313)}, pp.\  158--162. IET, 1989.

\bibitem[Sisson \& Fan(2011)Sisson and Fan]{sisson2011likelihood}
Sisson, S.~A. and Fan, Y.
\newblock \emph{Handbook of Markov Chain Monte Carlo}, chapter Likelihood-free
  {M}arkov chain {M}onte Carlo.
\newblock Chapman and Hall, 2011.

\bibitem[Sisson et~al.(2018)Sisson, Fan, and Beaumont]{sisson2018handbook}
Sisson, S.~A., Fan, Y., and Beaumont, M.
\newblock \emph{Handbook of Approximate {B}ayesian Computation}.
\newblock Chapman and Hall/CRC, 2018.

\bibitem[Yuret(2016)]{yuret2016knet}
Yuret, D.
\newblock Knet: beginning deep learning with 100 lines of julia.
\newblock In \emph{Machine Learning Systems Workshop at NIPS}, volume 2016,
  pp.\ ~5, 2016.

\bibitem[Zaheer et~al.(2017)Zaheer, Kottur, Ravanbakhsh, Poczos, Salakhutdinov,
  and Smola]{zaheer2017deep}
Zaheer, M., Kottur, S., Ravanbakhsh, S., Poczos, B., Salakhutdinov, R.~R., and
  Smola, A.~J.
\newblock Deep sets.
\newblock In \emph{Advances in Neural Information Processing Systems}, pp.\
  3391--3401, 2017.

\end{thebibliography}
\bibliographystyle{icml2019}

\clearpage
\onecolumn
\appendix
\section{Supplementary Material} \label{sec:suppl_mat}

\subsection{Approximate Bayesian computation rejection sampling}

\subsubsection{Settings for ABC rejection sampling ``reference table'' algorithm}
In section 2 of the main paper we denote with $x$ the ABC threshold.
For g-and-k and $\alpha$-stable models we consider for $x$ the $0.1$th percentile, and for AR(2) and MA(2) the $0.02$th percentile of all distances. The number of proposals for g-and-k and $\alpha$-stable models is $\tilde{N}=100,000$, and for  AR(2) and MA(2) $\tilde{N}=500,000$.

\subsubsection{The ABC distance function}
In all our inference attempts we always used ABC rejection sampling and only needed to change the method used to compute the summary statistics.
We employed the Mahalanobis distance
\begin{equation*}
\Delta(s^*, s^{\mathrm{obs}}) = \sqrt{(s^*-s^{\mathrm{obs}})^{\transpose}A(s^*-s^{\mathrm{obs}})},
\end{equation*}
where in our case $A$ is the identity matrix, except when using \textit{hand-picked} summary statistics for the g-and-k distribution, and in such case $A$ is a diagonal matrix with diagonal elements $1/w^2$, with $w$ a vector with entries $w = [0.22; 0.19; 0.53; 2.97; 1.90]$, as in \cite{picchini2017approximate}.

\subsection{Regularization}

We use early-stopping for all networks. The early-stopping method used is to train the network over $N$ epochs and then select the set of weights, out of the $N$ sets, that generated the lowest evaluation error.

\subsection{g-and-k distribution}

\begin{itemize}
    \item The full set of parameters for a g-and distribution is $[A,B,g,k,c]$, However, we follow the common practice of keeping $c$ fixed to  $c = 0.8$ and assume $B>0$ and $k\geq 0$ \cite{prangle2017gk}.
    \item Here is a procedure to simulate a single draw from the distribution: we first simulate a draw $z$ from a standard Gaussian distribution, $z \sim N(0,1)$, then we plug $z$ into
\begin{equation*}
    Q = A + B\cdot(1+c\cdot \tanh(g\cdot z/2))\cdot z\cdot(1+z^2)^k
\end{equation*}
and obtain a realization $Q$ from a g-and-k distribution.
    \item The network settings are presented in Table \ref{tab:networksettings_mlp_small_gandk}, \ref{tab:networksettings_mlp_large_gandk}, \ref{tab:networksettings_mlp_dnn_pre_gandk}, and \ref{tab:networksettings_deepsets_gandk};
    \item The number of weights for the different networks are presented in Table \ref{tab:gank_nbr_w};
    \item Values outside of the range $[-10,50]$ are considered to be outliers and these values are replaced (at random) with values inside the data range. The data cleaning scheme is applied to both the observed and generated data;
    \item When computing the empirical distribution function we evaluate this function over 100 equally spaced points between 0 and 50;
    \item Number of training observations: $5\cdot10^5$, $10^5$, $10^4$, and $10^3$. Evaluation data observations $5 \cdot 10^3$. 
\end{itemize}

\begin{table}[ht]
      \caption{g-and-k: Network settings for MLP small. \textcolor{white}{some extra text ;)} }
          \label{tab:networksettings_mlp_small_gandk}
      \centering
        \begin{tabular}{llll}
        Layer &  Dim. in & Dim. out  & Activation  \\ \midrule
        Input &  1000 & 25  & relu  \\
        Hidden 1 &  25 & 25  & relu \\
        Hidden 2 &  25 & 12  & relu  \\
        Output &  12 & 4  & linear  \\
        \bottomrule
        \end{tabular}
\end{table}
\begin{table}[ht]
      \caption{g-and-k: Network settings for MLP large. \textcolor{white}{some extra text ;)} }
          \label{tab:networksettings_mlp_large_gandk}
      \centering
        \begin{tabular}{llll}
        Layer &  Dim. in & Dim. out  & Activation  \\ \midrule
        Input &  1000 & 100  & relu  \\
        Hidden 1 &  100 & 100  & relu \\
        Hidden 2 &  100 & 50  & relu  \\
        Output &  50 & 4  & linear  \\
        \bottomrule
        \end{tabular}
\end{table}
\begin{table}[ht]
      \centering
        \caption{g-and-k: Network settings MLP pre}
        \label{tab:networksettings_mlp_dnn_pre_gandk}
        \begin{tabular}{llll}
        Layer &  Dim. in & Dim. out  & Activation  \\ \midrule
        Input &  100 & 100  & relu  \\
        Hidden 1 &  100 & 100  & relu \\
        Hidden 2 &  100 & 50  & relu  \\
        Output &  50 & 4  & linear  \\
        \bottomrule
        \end{tabular}

\end{table}

\begin{table}[ht]
        \centering
        \caption{g-and-k: Network settings for PEN-0}
        \label{tab:networksettings_deepsets_gandk}
        \begin{tabular}{llll}
        $\phi$ network &    &    &   \\ \midrule
        Layer &  Dim. in & Dim. out  & Activation  \\ \midrule
        Input &  1 & 100  & relu  \\
        Hidden 1 &  100 & 50  & relu \\
        Output &  50 & 10  & linear  \\  \rule{0pt}{4ex}

        $\rho$ network &    &    &   \\ \midrule
        Layer &  Dim. in & Dim. out  & Activation  \\ \midrule
        Input &  10 & 100  & relu  \\
        Hidden 1 &  100 & 100  & relu \\
        Hidden 2 &  100 & 50  & relu  \\
        Output &  50 & 4  & linear  \\
        \bottomrule
        \end{tabular}
\end{table}

\begin{table}[ht]
\caption{g-and-k: Number of weights for the different networks }
          \label{tab:gank_nbr_w}
      \centering
        \begin{tabular}{ll}
        Network &  \# weights \\ \midrule
        MLP small & 26039 \\
        MLP large &  115454 \\
        MLP pre & 25454 \\
        PEN-0 & 22214 \\
        \bottomrule
        \end{tabular}
\end{table}

\subsection{$\alpha$-stable distribution}

\begin{itemize}

\item The characteristic function $\varphi(x)$ for the $\alpha$-stable distribution is given by \cite{ong2018variational}

\begin{align*}
    \varphi(x) = \begin{cases} \exp \Big( i \delta t - \gamma^{\alpha} \lvert t \rvert^{\alpha} \big(1 + i \beta \tan \frac{\pi \alpha}{2}\text{sgn}(t)( \lvert \gamma t \rvert^{1-\alpha}-1)\big) \Big), \ \alpha \neq 1, \\
    \exp\Big(  i \delta t - \gamma \lvert t \rvert \big(1 + i \beta \frac{2}{\pi}\text{sgn}(t)\log(\gamma  \lvert t \rvert )\big) \Big), \ \alpha = 1,
    \end{cases}
\end{align*}
where $\text{sgn}$ is the sign function, i.e.
\begin{align*}
    \text{sgn}(t) = \begin{cases}  -1 \ &\text{if}  \ t < 0, \\
    0 \ &\text{if}  \ t = 0, \\
    1 \ &\text{if}  \ t > 0. \end{cases}
\end{align*}

    \item The network settings are presented in Table \ref{tab:networksettings_mlp_small_alpha}, \ref{tab:networksettings_mlp_large_alpha}, \ref{tab:networksettings_mlp_dnn_pre_alpha}, and \ref{tab:networksettings_deepsets_alpha};
    \item The number of weights for the different networks are presented in Table \ref{tab:alpha_nbr_w};
    \item Values outside of the range $[-10,50]$ are considered to be outliers and these values are replaced (at random) with values inside the data range. The data cleaning scheme is applied to both the observed and generated data;
    \item All data sets are standardized using the ``robust scalar'' method, i.e. each data point $y_i$ is standardized according to

    \begin{align*}
        \frac{y_i + Q_1(y)}{Q_3(y)-Q_1(y)}
    \end{align*}
    where $Q_1$ and $Q_3$ are the first and third quantiles respectively;

    \item When computing the empirical distribution function we evaluate this function over 100 equally spaced points between -10 and 100;
    \item The root-mean-squared error (RMSE) is computed as

    \begin{align*}
        \text{RMSE} = \sqrt{\frac{1}{R} \sum_{i=1}^{R}\{ (\hat{\theta}^1_i - \theta^1)^2 + (\hat{\theta}^2_i - \theta^2)^2 + (\hat{\theta}^3_i - \theta^3)^2 + (\hat{\theta}^4_i - \theta^4)^2\}},
    \end{align*}
    where $\theta = [\theta^1, \theta^2, \theta^3, \theta^4]$ are ground-truth parameter values and  $[\hat{\theta}^1_i, \hat{\theta}^2_i, \hat{\theta}^3_i, \hat{\theta}^4_i]_{1 \leq i \leq R}$ are ABC posterior means. $R$ is the number of independent repetitions of the inference procedure;

    \item Number of training observations: $5\cdot10^5$, $10^5$, $10^4$, and $10^3$. Evaluation data observations $5 \cdot 10^3$. 

\end{itemize}

\begin{table}[ht]
      \caption{$\alpha$-stable: Network settings for MLP small. \textcolor{white}{some extra text ;)} }
          \label{tab:networksettings_mlp_small_alpha}
      \centering
        \begin{tabular}{llll}
        Layer &  Dim. in & Dim. out  & Activation  \\ \midrule
        Input &  1002 & 25  & relu  \\
        Hidden 1 &  25 & 25  & relu \\
        Hidden 2 &  25 & 12  & relu  \\
        Output &  12 & 4  & linear  \\
        \bottomrule
        \end{tabular}
\end{table}
\begin{table}[ht]

      \caption{$\alpha$-stable: Network settings for MLP large. \textcolor{white}{some extra text ;)} }
          \label{tab:networksettings_mlp_large_alpha}
      \centering
        \begin{tabular}{llll}
        Layer &  Dim. in & Dim. out  & Activation  \\ \midrule
        Input &  1002 & 100  & relu  \\
        Hidden 1 &  100 & 100  & relu \\
        Hidden 2 &  100 & 50  & relu  \\
        Output &  50 & 4  & linear  \\
        \bottomrule
        \end{tabular}
\end{table}
\begin{table}[ht]
      \centering
        \caption{$\alpha$-stable: Network settings MLP pre.}
        \label{tab:networksettings_mlp_dnn_pre_alpha}
        \begin{tabular}{llll}
        Layer &  Dim. in & Dim. out  & Activation  \\ \midrule
        Input &  100 & 100  & relu  \\
        Hidden 1 &  100 & 100  & relu \\
        Hidden 2 &  100 & 50  & relu  \\
        Output &  50 & 4  & linear  \\
        \bottomrule
    \end{tabular}
\end{table}
\begin{table}[ht]
        \centering
        \caption{$\alpha$-stable: Network settings for PEN-0.}
        \label{tab:networksettings_deepsets_alpha}
        \begin{tabular}{llll}
        $\phi$ network &    &    &   \\ \midrule
        Layer &  Dim. in & Dim. out  & Activation  \\ \midrule
        Input &  1 & 100  & relu  \\
        Hidden 1 &  100 & 50  & relu \\
        Output &  50 & 20  & linear  \\  \rule{0pt}{4ex}

        $\rho$ network &    &    &   \\ \midrule
        Layer &  Dim. in & Dim. out  & Activation  \\ \midrule
        Input &  22 & 100  & relu  \\
        Hidden 1 &  100 & 100  & relu \\
        Hidden 2 &  100 & 50  & relu  \\
        Output &  50 & 4  & linear  \\
        \bottomrule
        \end{tabular}
\end{table}

\begin{table}[ht]
\caption{$\alpha$-stable: Number of weights for the different networks }
          \label{tab:alpha_nbr_w}
      \centering
        \begin{tabular}{ll}
        Network &  \# weights \\ \midrule
        MLP small & 26089  \\
        MLP large & 115654 \\
        MLP pre &  25454 \\
        PEN-0 &  23924 \\
        \bottomrule
        \end{tabular}
\end{table}

\subsection{Autoregressive time series model}

\begin{itemize}
\item The network settings are presented in Table \ref{tab:networksettings_ar2_small_dnn}, \ref{tab:networksettings_ar2_large_dnn},  \ref{tab:networksettings_ar2_mlp_pen0}, and \ref{tab:networksettings_ar2_mlp_pen2};
  \item The number of weights for the different networks are presented in Table \ref{tab:ar_nbr_w};
 \item Number of training observations: $10^6$, $10^5$, $10^4$, and $10^3$. Evaluation data observations $10^4$. 
\end{itemize}

\begin{table}[ht]
      \caption{AR(2): Network settings for MLP small.}
     \label{tab:networksettings_ar2_small_dnn}
      \centering
        \begin{tabular}{llll}
        Layer &  Dim. in & Dim. out  & Activation  \\ \midrule
        Input &  100 & 55  & relu  \\
        Hidden 1 &  55 & 55  & relu \\
        Hidden 2 &  55 & 25  & relu  \\
        Output &  25 & 2  & linear  \\
        \bottomrule
        \end{tabular}
\end{table}
\begin{table}[ht]
      \caption{AR(2): Network settings for MLP large.}
     \label{tab:networksettings_ar2_large_dnn}
      \centering
        \begin{tabular}{llll}
        Layer &  Dim. in & Dim. out  & Activation  \\ \midrule
        Input &  100 & 100  & relu  \\
        Hidden 1 &  100 & 100  & relu \\
        Hidden 2 &  100 & 50  & relu  \\
        Output &  50 & 2  & linear  \\
        \bottomrule
        \end{tabular}
\end{table}
\begin{table}[ht]
        \centering
        \caption{AR(2): Network settings for PEN-0.}
        \label{tab:networksettings_ar2_mlp_pen0}
        \begin{tabular}{llll}
        $\phi$ network &    &    &   \\ \midrule
        Layer &  Dim. in & Dim. out  & Activation  \\ \midrule
        Input &  1 & 100  & relu  \\
        Hidden 1 &  100 & 50  & relu \\
        Output &  50 & 10  & linear  \\ \rule{0pt}{4ex}

        $\rho$ network &    &    &   \\ \midrule
        Layer &  Dim. in & Dim. out  & Activation  \\ \midrule
        Input &  10 & 50  & relu  \\
        Hidden 1 &  50 & 50  & relu \\
        Hidden 2 &  50 & 20  & relu  \\
        Output &  20 & 2  & linear  \\
        \bottomrule
        \end{tabular}
\end{table}
\begin{table}[ht]
        \centering
        \caption{AR(2): Network settings for PEN-2.}
        \label{tab:networksettings_ar2_mlp_pen2}
        \begin{tabular}{llll}
        $\phi$ network &    &    &   \\ \midrule
        Layer &  Dim. in & Dim. out  & Activation  \\ \midrule
        Input &  3 & 100  & relu  \\
        Hidden 1 &  100 & 50  & relu \\
        Output &  50 & 10  & linear  \\ \rule{0pt}{4ex}

        $\rho$ network &    &    &   \\ \midrule
        Layer &  Dim. in & Dim. out  & Activation  \\ \midrule
        Input &  12 & 50  & relu  \\
        Hidden 1 &  50 & 50  & relu \\
        Hidden 2 &  50 & 20  & relu  \\
        Output &  20 & 2  & linear  \\
        \bottomrule
        \end{tabular}
\end{table}

\begin{table}[ht]
\caption{AR(2): Number of weights for the different networks }
          \label{tab:ar_nbr_w}
      \centering
        \begin{tabular}{ll}
        Network &  \# weights \\ \midrule
        MLP small &  10087 \\
        MLP large &  25352 \\
        PEN-0 &  9922 \\
        PEN-2 &  10222 \\
        \bottomrule
        \end{tabular}
\end{table}

\subsection{Moving average time series with observational noise model}

\begin{itemize}
\item The network settings are presented in Table \ref{tab:networksettings_ma2_mlp_small}, \ref{tab:networksettings_ma2_mlp_large}, \ref{tab:networksettings_ma2_mlp_pen0}, and \ref{tab:networksettings_ma2_mlp_pen10};
  \item The number of weights for the different networks are presented in Table \ref{tab:ma_nbr_w};
 \item Number of training observations: $10^6$, $10^5$, $10^4$, and $10^3$. Evaluation data observations $5 \cdot 10^5$.  
\end{itemize}

\begin{table}[ht]
      \caption{MA(2): Network settings for MLP small.}
    \label{tab:networksettings_ma2_mlp_small}
      \centering
        \begin{tabular}{llll}
        Layer &  Dim. in & Dim. out  & Activation  \\ \midrule
        Input &  100 & 60  & relu  \\
        Hidden 1 &  60 & 60  & relu \\
        Hidden 2 &  60 & 25  & relu  \\
        Output &  25 & 2  & linear  \\
        \bottomrule
        \end{tabular}
\end{table}
\begin{table}[ht]
      \caption{MA(2): Network settings for MLP large.}
    \label{tab:networksettings_ma2_mlp_large}
      \centering
        \begin{tabular}{llll}
        Layer &  Dim. in & Dim. out  & Activation  \\ \midrule
        Input &  100 & 100  & relu  \\
        Hidden 1 &  100 & 100  & relu \\
        Hidden 2 &  100 & 50  & relu  \\
        Output &  50 & 2  & linear  \\
        \bottomrule
        \end{tabular}
\end{table}
\begin{table}[ht]
        \centering
        \caption{MA(2): Network settings for PEN-0.}
        \label{tab:networksettings_ma2_mlp_pen0}
        \begin{tabular}{llll}
        $\rho$ network &    &    &   \\ \midrule
        Layer &  Dim. in & Dim. out  & Activation  \\ \midrule
        Input &  1 & 100  & relu  \\
        Hidden 1 &  100 & 50  & relu \\
        Hidden 2 &  50 & 10  & relu  \\

        $\phi$ network &    &    &   \\ \midrule
        Layer &  Dim. in & Dim. out  & Activation  \\ \midrule
        Input &  10 & 50  & relu  \\
        Hidden 1 &  50 & 50  & relu \\
        Hidden 2 &  50 & 20  & relu  \\
        Output &  20 & 2  & linear  \\
        \bottomrule
        \end{tabular}
\end{table}
\begin{table}[ht]
        \centering
        \caption{MA(2): Network settings for PEN-10}
        \label{tab:networksettings_ma2_mlp_pen10}
        \begin{tabular}{llll}
        $\rho$ network &    &    &   \\ \midrule
        Layer &  Dim. in & Dim. out  & Activation  \\ \midrule
        Input &  11 & 100  & relu  \\
        Hidden 1 &  100 & 50  & relu \\
        Hidden 2 &  50 & 10  & relu  \\

        $\phi$ network &    &    &   \\ \midrule
        Layer &  Dim. in & Dim. out  & Activation  \\ \midrule
        Input &  20 & 50  & relu  \\
        Hidden 1 &  50 & 50  & relu \\
        Hidden 2 &  50 & 20  & relu  \\
        Output &  20 & 2  & linear  \\
        \bottomrule
        \end{tabular}
\end{table}

\begin{table}[ht]
\caption{MA(2): Number of weights for the different networks }
          \label{tab:ma_nbr_w}
      \centering
        \begin{tabular}{ll}
        Network &  \# weights \\ \midrule
        MLP small & 11297   \\
        MLP large & 25352   \\
        PEN-0 & 9922 \\
        PEN-10 & 11422 \\
        \bottomrule
        \end{tabular}
\end{table}

\end{document}